%% file: main.tex
\documentclass{article} 
\usepackage{xcolor}
\definecolor{mydarkblue}{rgb}{0,0.08,0.45}

 \usepackage[nonatbib, final]{neurips_2019}

\usepackage[utf8]{inputenc} 
\usepackage[T1]{fontenc}    
\usepackage[colorlinks=true,
    linkcolor=mydarkblue,
    citecolor=mydarkblue,
    filecolor=mydarkblue,
    urlcolor=mydarkblue]{hyperref} 
\usepackage{url}            
\usepackage{booktabs}       
\usepackage{amsfonts}       
\usepackage{nicefrac}       
\usepackage{microtype}      
\usepackage{amsmath,amssymb,amsthm,amsfonts,latexsym}        
\usepackage{graphicx}
\usepackage{bm}
\usepackage{subfigure}
\usepackage{color}
\usepackage{appendix}
\usepackage{enumitem}
\usepackage{environ}        
\usepackage{etoolbox}       
\usepackage{setspace}
\usepackage{algpseudocode}
\usepackage{algorithm}

\usepackage{todonotes}

\title{Painless Stochastic Gradient:\\ Interpolation, Line-Search, and Convergence Rates}
\author{
  Sharan Vaswani \\
  Mila, Universit\'e de Montr\'eal \\
  \And
  Aaron Mishkin \\
  University of British Columbia \\
  \AND
  Issam Laradji \\
  University of British Columbia\\ 
  Element AI \\
  \And
  Mark Schmidt \\
  University of British Columbia, 1QBit \\
  CCAI Affiliate Chair (Amii)\\
  \And
  Gauthier Gidel \\
  Mila, Université de Montr\'eal \\
  Element AI \\
  \And
  Simon Lacoste-Julien$^\dagger$ \\
  Mila, Université de Montr\'eal \\
}
\include{symbols}

\usepackage[compact]{titlesec}
\titlespacing{\section}{10pt}{*0}{*0}
\titlespacing{\subsection}{12pt}{*0}{*0}

\begin{document}
\maketitle

\setstretch{1}

\input{Abstract}
\input{Introduction}
\input{Background}
\input{SGD-c}
\input{SGD-nc}
\input{SEG}
\input{Practical}
\input{Experiments}
\input{Conclusion}
\input{Acknowledgements}
\bibliographystyle{plain}
\bibliography{ref}

\clearpage
\newgeometry{margin=0.65in}
\appendix
\input{App-Proofs-common}
\input{App-Proofs-SGD-convex}
\input{App-Proofs-SGD-nonconvex-correct}
\input{App-Proofs-SEG}

\input{App-Proofs-SGD-nc-additional}
\input{App-Experimental-Details}
\input{App-Additional-Results}
\input{App-Algorithms}
\end{document}

%% file: symbols.tex
\newcommand{\x}{w}
\newcommand{\xk}{w_{k}}

\newcommand{\mxt}{\bar{w}_{T}}
\newcommand{\xkk}{w_{k+1}}
\newcommand{\xopt}{w^{*}}
\newcommand{\xkh}{w^\prime_{k}}

\newcommand{\grad}[1]{\nabla f(#1)}

\newcommand{\norm}[1]{\left\|#1\right\|}
\newcommand{\normsq}[1]{\left\|#1\right\|^{2}}

\newcommand{\E}{\mathbb{E}}

\newcommand{\fk}{f_{ik}}
\newcommand{\fj}{f_{i}}
\newcommand{\etak}{\eta_{k}}
\newcommand{\muk}{\mu_{ik}}

\newcommand{\Lk}{L_{ik}}

\newcommand{\gradk}[1]{\nabla f_{ik}(#1)}
\newcommand{\gradi}[1]{\nabla f_{i}(#1)}
\newcommand{\opk}[1]{F_{ik}(#1)}

\DeclareMathOperator*{\argmin}{arg\,min}

\newtheorem{theorem}{Theorem}
\newtheorem{lemma}{Lemma}

\newcommand{\linenumber}{\addtocounter{equation}{1}\tag{\theequation}}

\newcommand{\rbr}[1]{\left(#1\right)}
\newcommand{\sbr}[1]{\left[#1\right]}
\newcommand{\cbr}[1]{\left\{#1\right\}}
\newcommand{\abr}[1]{\left\langle#1\right\rangle}
\newcommand{\Lmax}{L_{\text{max}}}
\newcommand{\etamin}{\eta_{\text{min}}}
\newcommand{\etamax}{\eta_{\text{max}}}
\newcommand{\R}{\mathbb{R}}

%% file: Abstract.tex
\begin{abstract}
Recent works have shown that stochastic gradient descent (SGD) achieves the fast convergence rates of full-batch gradient descent for over-parameterized models satisfying certain interpolation conditions. However, the step-size used in these works depends on unknown quantities and SGD's practical performance heavily relies on the choice of this step-size. We propose to use line-search techniques to automatically set the step-size when training models that can interpolate the data. In the interpolation setting, we prove that SGD with a stochastic variant of the classic Armijo line-search attains the deterministic convergence rates for both convex and strongly-convex functions. Under additional assumptions, SGD with Armijo line-search is shown to achieve fast convergence for non-convex functions. Furthermore, we show that stochastic extra-gradient with a Lipschitz line-search attains linear convergence for an important class of non-convex functions and saddle-point problems satisfying interpolation. To improve the proposed methods' practical performance, we give heuristics to use larger step-sizes and acceleration. We compare the proposed algorithms against numerous optimization methods on standard classification tasks using both kernel methods and deep networks. The proposed methods result in competitive performance across all models and datasets, while being robust to the precise choices of hyper-parameters. For multi-class classification using deep networks, SGD with Armijo line-search results in both faster convergence and better generalization. 
\end{abstract}

%% file: Introduction.tex
\section{Introduction}
\label{sec:introduction}
Stochastic gradient descent (SGD) and its variants~\cite{duchi2011adaptive,zeiler2012adadelta,kingma2014adam,tieleman2012lecture,schmidt2017minimizing,johnson2013accelerating,defazio2014saga} are the preferred optimization methods in modern machine learning. They only require the gradient for one training example (or a small ``mini-batch'' of examples) in each iteration and thus can be used with large datasets. These first-order methods have been particularly successful for training highly-expressive, over-parameterized models such as non-parametric regression~\cite{liang2018just,belkin2019does} and deep neural networks~\cite{bengio2012practical,zhang2016understanding}. However, the practical efficiency of stochastic gradient methods is adversely affected by two challenges: (i) their performance heavily relies on the choice of the step-size (``learning rate'')~\cite{bengio2012practical, schaul2013no} and (ii) their slow convergence compared to methods that compute the full gradient (over all training examples) in each iteration~\cite{nesterov2013introductory}. 

Variance-reduction (VR) methods~\cite{schmidt2017minimizing,johnson2013accelerating,defazio2014saga} are relatively new variants of SGD that improve its slow convergence rate. These methods exploit the finite-sum structure of typical loss functions arising in machine learning, achieving both the low iteration cost of SGD and the fast convergence rate of deterministic methods that compute the full-gradient in each iteration. Moreover, VR makes setting the learning rate easier and there has been work exploring the use of line-search techniques for automatically setting the step-size for these methods~\cite{schmidt2017minimizing,schmidt2015non,tan2016barzilai,shang2018guaranteed}. These methods have resulted in impressive performance on a variety of problems. However, the improved performance comes at the cost of additional memory~\cite{schmidt2017minimizing} or computational~\cite{johnson2013accelerating, defazio2014saga} overheads, making these methods less appealing when training high-dimensional models on large datasets. Moreover, in practice VR methods do not tend to converge faster than SGD on over-parameterized models~\cite{defazio2018ineffectiveness}. 

Indeed, recent works~\cite{vaswani2019fast,ma2018power,bassily2018exponential,liu2018mass, cevher2018linear, jain2017accelerating, schmidt2013fast} have shown that when training over-parameterized models, classic SGD with a constant step-size and \emph{without VR} can achieve the convergence rates of full-batch gradient descent. These works assume that the model is expressive enough to \emph{interpolate} the data. The interpolation condition is satisfied for models such as non-parametric regression~\cite{liang2018just,belkin2019does}, over-parametrized deep neural networks~\cite{zhang2016understanding}, boosting~\cite{schapire1998boosting}, and for linear classifiers on separable data. However, the good performance of SGD in this setting relies on using the proposed constant step-size, which depends on problem-specific quantities not known in practice. On the other hand, there has been a long line of research on techniques to automatically set the step-size for classic SGD. These techniques include using meta-learning procedures to modify the main stochastic algorithm~\cite{baydin2017online,yu2006fast,schraudolph1999local,almeida1998parameter,plagianakos2001learning,yu2006fast,shao2000rates}, heuristics to adjust the learning rate on the fly~\cite{kushner1995stochastic, deylon1993accelerated, schaul2013no, schoenauer2017stochastic}, and recent adaptive methods inspired by online learning~\cite{duchi2011adaptive,zeiler2012adadelta,kingma2014adam,reddi2019convergence, orabona2017training,rolinek2018l4, luo2019adaptive}. However, none of these techniques have been proved to achieve the fast convergence rates that we now know are possible in the over-parametrized setting.

In this work, we use classical line-search methods~\cite{nocedal2006numerical} to automatically set the step-size for SGD when training over-parametrized models. Line-search is a standard technique to adaptively set the step-size for deterministic methods that evaluate the full gradient in each iteration. These methods make use of additional function/gradient evaluations to characterize the function around the current iterate and adjust the magnitude of the descent step. The additional noise in SGD complicates the use of line-searches in the general stochastic setting and there have only been a few attempts to address this. Mahsereci et al.~\cite{mahsereci2017probabilistic} define a Gaussian process model over probabilistic Wolfe conditions and use it to derive a termination criterion for the line-search. The convergence rate of this procedure is not known, and experimentally we found that our proposed line-search technique is simpler to implement and more robust. Other authors~\cite{friedlander2012hybrid,byrd2012sample, de2016big, paquette2018stochastic,krejic2013line} use a line-search termination criteria that requires function/gradient evaluations averaged over multiple samples. However, in order to achieve convergence, the number of samples required per iteration (the ``batch-size'') increases progressively, losing the low per iteration cost of SGD. Other work~\cite{blanchet2019convergence,gratton2017complexity} exploring trust-region methods assume that the model is sufficiently accurate, which is not guaranteed in the general stochastic setting. In contrast to these works, our line-search procedure does not consider the general stochastic setting and is designed for models that satisfy interpolation; it achieves fast rates in the over-parameterized regime without the need to manually choose a step-size or increase the batch size.

We make the following contributions: in Section~\ref{sec:sgd-c} we prove that, under interpolation, SGD with a stochastic variant of the Armijo line-search attains the convergence rates of full-batch gradient descent in both the convex and strongly-convex settings. We achieve these rates under weaker assumptions than the prior work~\cite{vaswani2019fast} and \emph{without} the explicit knowledge of problem specific constants. We then consider minimizing non-convex functions satisfying interpolation~\cite{bassily2018exponential, vaswani2019fast}. Previous work~\cite{bassily2018exponential} proves that constant step-size SGD achieves a linear rate for non-convex functions satisfying the PL inequality~\cite{polyak1963gradient, karimi2016linear}. SGD is further known to achieve deterministic rates for general non-convex functions under a stronger assumption on the growth of the stochastic gradients~\cite{schmidt2013fast, vaswani2019fast}. Under this assumption and an upper bound (that requires knowledge of the ``Lipschitz'' constant) on the maximum step size, we prove that SGD with Armijo line-search can achieve the deterministic rate for general non-convex functions (Section~\ref{sec:sgd-nc}). Note that these are the first convergence rates for SGD with line-search in the interpolation setting for both convex and non-convex functions. 

Moving beyond SGD, in Section~\ref{sec:seg} we consider the stochastic extra-gradient (SEG) method~\cite{korpelevich1976extragradient, nemirovski2004prox, juditsky2011solving,iusem2017extragradient, gidel2018variational} used to solve general variational inequalities~\cite{harker1990finite}. These problems encompass both convex minimization and saddle point problems arising in robust supervised learning~\cite{ben2009robust,wen2014robust} and learning with non-separable losses or regularizers~\cite{joachims2005support,bach2012optimization}. In the interpolation setting, we show that a variant of SEG~\cite{gidel2018variational} with a ``Lipschitz'' line-search convergences linearly when minimizing an important class of non-convex functions~\cite{li2017convergence, kleinberg2018alternative, soltanolkotabi2018theoretical,sun2016guaranteed,chen2015solving} satisfying the restricted secant inequality (RSI). Moreover, in Appendix~\ref{app:seg}, we prove that the same algorithm results in linear convergence for both strongly convex-concave and bilinear saddle point problems satisfying interpolation. 

In Section~\ref{sec:practical}, we give heuristics to use large step-sizes and integrate acceleration with our line-search techniques, which improves practical performance of the proposed methods. We compare our algorithms against numerous optimizers~\cite{kingma2014adam, duchi2011adaptive,rolinek2018l4, mahsereci2017probabilistic,orabona2017training,luo2019adaptive} on a synthetic matrix factorization problem (Section~\ref{sec:experiments-synthetic}), convex binary-classification problems using radial basis function (RBF) kernels (Section~\ref{sec:experiments-kernels}), and non-convex multi-class classification problems with deep neural networks (Section~\ref{sec:experiments-deep}). We observe that when interpolation is (approximately) satisfied, the proposed methods are robust and have competitive performance across models and datasets. Moreover, SGD with Armijo line-search results in both faster convergence and better generalization performance for classification using deep networks. Finally, in Appendix~\ref{app:experiments-games}, we evaluate SEG with line-search for synthetic bilinear saddle point problems. The code to reproduce our results can be found at \url{https://github.com/IssamLaradji/sls}.

One of the most important special cases where interpolation is often satisfied is training deep neural networks. The work of Paul Tseng~\cite{tseng1998incremental} is the earliest work that we are aware of that considers training neural networks with SGD and a line search. Tseng further shows convergence and convergence rates for SGD in an interpolation setting, but his method was never widely-adopted and is more complicated than the simple stochastic Armijo method analyzed in this work. Several recent works have evaluated the empirical performance of using Armijo-style line-searches in an SGD setting for training deep neural networks, including exploring combinations with momentum/acceleration and quasi-Newton methods~\cite{bollapragada2018progressive,truong2018backtracking}. The work of Truong and Nguyen in particular shows strong empirical performance for several benchmark deep learning problems, but their theory only considers deterministic settings and does not apply to SGD.

In concurrent work to ours, Berrada et al.~\cite{berrada2019training} propose adaptive step-sizes for SGD on convex, finite-sum loss functions under an $\epsilon$-interpolation condition. Unlike our approach, $\epsilon$-interpolation requires knowledge of a lower bound on the global minimum and only guarantees approximate convergence to a stationary point. Moreover, in order to obtain linear convergence rates, they assume $\mu$-strong-convexity of \emph{each} individual function. This assumption with $\epsilon$-interpolation reduces the finite-sum optimization to minimization of \emph{any single} function in the finite sum. 

%% file: Background.tex
\section{Assumptions}
\label{sec:background}
We aim to minimize a differentiable function $f$ assuming access to noisy stochastic gradients of the function. We focus on the common machine learning setting where the function $f$ has a \emph{finite-sum structure} meaning that $f(\x) = \frac{1}{n} \sum_{i = 1}^{n} \fj(\x)$. Here $n$ is equal to the number of points in the training set and the function $\fj$ is the loss function for the  training point $i$. Depending on the model, $f$ can either be strongly-convex, convex, or non-convex. We assume that $f$ is lower-bounded by some value $f^*$ and that $f$ is $L$-smooth~\cite{nemirovski2009robust} implying that the gradient $\nabla f$ is $L$-Lipschitz continuous. 

We assume that the model is able to interpolate the data and use this property to derive convergence rates. Formally, interpolation requires that the gradient with respect to \emph{each} point converges to zero at the optimum, implying that if the function $f$ is minimized at $\xopt$ and thus $\grad{\xopt} = 0$, then for all functions $\fj$ we have that $\gradi{\xopt} = 0$. For example, interpolation is exactly satisfied when using a linear model with the squared hinge loss for binary classification on linearly separable data.

%% file: SGD-c.tex
\section{Stochastic Gradient Descent for Convex Functions}
\label{sec:sgd-c}
Stochastic gradient descent (SGD) computes the gradient of the loss function corresponding to one or a mini-batch of randomly (typically uniformly) chosen training examples $i_k$ in iteration $k$. It then performs a descent step as $\xkk = \xk - \etak \gradk{\xk}$, where $\xkk$ and $\xk$ are the SGD iterates, $\etak$ is the step-size and $\gradk{\cdot}$ is the (average) gradient of the loss function(s) chosen at iteration $k$. Each stochastic gradient $\gradk{\x}$ is assumed to be unbiased, implying that $\E_i \left[ \gradi{\x} \right] = \grad{\x}$ for all $\x$. We now describe the Armijo line-search method to set the step-size in each iteration.
\subsection{Armijo line-search}
\label{sec:c-ls}
Armijo line-search~\cite{armijo1966minimization} is a standard method for setting the step-size for gradient descent in the deterministic setting~\cite{nocedal2006numerical}. We adapt it to the stochastic case as follows: at iteration $k$, the Armijo line-search selects a step-size satisfying the following condition: 
\begin{align}
\fk \left(\xk - \etak \gradk{\xk} \right) & \leq \fk(\xk) -  c \cdot \etak \normsq{ \gradk{\x_k}  }.  \label{eq:c-ls} 
\end{align}
Here, $c > 0$ is a hyper-parameter. Note that the above line-search condition uses the function and gradient values \emph{of the mini-batch} at the current iterate $\xk$. Thus, compared to SGD, checking this condition only makes use of additional mini-batch function (and not gradient) evaluations. In the context of deep neural networks, this corresponds to extra forward passes on the mini-batch.

In our theoretical results, we assume that there is a maximum step-size $\eta_\text{max}$ from which the line-search starts in \emph{each} iteration $k$ and that we choose the largest step-size $\eta_k$ (less than or equal to $\eta_\text{max}$) satisfying~\eqref{eq:c-ls}. In practice, backtracking line-search is a common way to ensure that  Equation~\ref{eq:c-ls} is satisfied. Starting from $\eta_\text{max}$, backtracking iteratively decreases the step-size by a constant factor $\beta$ until the line-search succeeds (see Algorithm~\ref{alg:SGD_Armijo}). Suitable strategies for \emph{resetting} the step-size can avoid backtracking in the majority of iterations and make the step-size selection procedure efficient. We describe such strategies in Section~\ref{sec:practical}. With resetting, we required (on average) only one additional forward pass on the mini-batch per iteration when training a standard deep network model (Section~\ref{sec:experiments-deep}). Empirically, we observe that the algorithm is robust to the choice of both $c$ and $\eta_\text{max}$;  setting $c$ to a small constant and $\eta_\text{max}$ to a large value consistently results in good performance. 

We bound the chosen step-size in terms of the properties of the function(s) selected in iteration $k$. 
\begin{lemma}
The step-size $\etak$ returned by the Armijo line-search and constrained to lie in the $(0, \eta_{\text{max}}]$ range satisfies the following inequality,
\begin{align}
\etak \geq \min \left\{\frac{2 \; (1-c)}{\Lk}, \eta_\text{max} \right\},
\label{eq:eta-bounds}
\end{align}
where $\Lk$ is the Lipschitz constant of $\nabla f_{i_k}$. \label{lemma:eta-bounds}
\end{lemma}
\vspace*{-0.8ex}
The proof is in Appendix~\ref{app:lemmas} and follows the deterministic case~\cite{nocedal2006numerical}. Note that Equation~\eqref{lemma:eta-bounds} holds for all smooth functions (for small-enough $\eta_k$), does not require convexity, and guarantees backtracking line-search will terminate at a non-zero step-size. The parameter $c$ controls the ``aggressiveness'' of the algorithm; small $c$ values encourage a larger step-size. 
For a sufficiently large $\eta_\text{max}$ and $c \leq 1/2$, the step-size is at least as large as $1/\Lk$, which is the constant step-size used in the interpolation setting~\cite{vaswani2019fast, schmidt2013fast}. 
In practice, we expect these larger step-sizes to result in improved performance.  In Appendix~\ref{app:lemmas}, we also give upper bounds on $\etak$ if the function $f_{i_k}$ satisfies the  Polyak-Lojasiewicz (PL) inequality~\cite{polyak1963gradient,karimi2016linear} with constant $\muk$. PL is a weaker condition than strong-convexity and does not require convexity. In this case, $\etak$ is upper-bounded by the minimum of $\eta_\text{max}$ and $1/(2 c \cdot \muk)$. If we use a backtracking line-search that multiplies the step-size by $\beta$ until~\eqref{eq:c-ls} holds, the step-size will be smaller by at most a factor of $\beta$ (we do not include this dependence in our results).

\subsection{Convergence rates}
In this section, we characterize the convergence rate of SGD with Armijo line-search in the strongly-convex and convex cases. The theorems below are proved in Appendix~\ref{app:sc} and Appendix~\ref{app:c} respectively. 
\begin{theorem}[Strongly-Convex]
Assuming (a) interpolation, (b) $L_i$-smoothness, (c) convexity of $f_i$'s, and (d) $\mu$ strong-convexity of $f$, SGD with Armijo line-search with $c = \nicefrac{1}{2}$ in Eq.~\ref{eq:c-ls} achieves the rate:
\begin{align*}
\E \left[ \normsq{\x_{T} - \xopt} \right] & \leq \max \left\{ 
\left(1 - \frac{\bar \mu}{L_{\text{max}}} \right), \left(1 - \bar \mu \; \eta_{\text{max}} \right) \right\}^{T} \normsq{\x_{0} - \xopt}.
\end{align*}
\label{thm:sc}
Here $\bar{\mu} = \sum_{i = 1}^{n} \mu_i / n$ is the average strong-convexity of the finite sum and $L_{max} = \max_{i} L_{i}$ is the maximum smoothness constant in the $f_i$'s.
\end{theorem}
In contrast to the previous results~\cite{vaswani2019fast, schmidt2013fast, ma2018power} that depend on $\mu$, the above linear rate depends on $\bar{\mu} \leq \mu$.
Note that unlike Berrada et al.~\cite{berrada2019training}, we do not require that \emph{each} $f_i$ is strongly convex, but for $\bar{\mu}$ to be non-zero we still require that \emph{at least one} of the $f_i$'s is strongly-convex. 
\begin{theorem}[Convex]
Assuming (a) interpolation, (b) $L_{i}$-smoothness and (c) convexity of $f_{i}$'s, SGD with Armijo line-search for all $c > 1/2$ in Equation~\ref{eq:c-ls} and iterate averaging achieves the rate:
\begin{align*}
\E \left[ f(\mxt) - f(\xopt) \right] & \leq \frac{c \cdot \max\left\{\frac{L_{\text{max}}}{2 \; (1-c)}, \frac{1}{\eta_{\text{max}}} \right\}}{(2c - 1) \; T} \normsq{\x_{0} - \xopt}. 
\end{align*}
Here, $\mxt = \frac{\left[ \sum_{i = 1}^{T} \x_{i} \right]}{T}$ is the averaged iterate after $T$ iterations and $L_{\text{max}} = \max_{i} L_i$. 
\label{thm:c}
\end{theorem}
\vspace*{-1.2ex}
In particular, setting $c = 2/3$ implies that $\E \left[f(\mxt) - f(\xopt) \right] \leq \frac{\max\left\{3 \; L_{\text{max}}, \frac{2}{\eta_{\text{max}}} \right\}}{T} \normsq{\x_{0} - \xopt}$. These are the first rates for SGD with line-search in the interpolation setting and match the corresponding rates for full-batch gradient descent on strongly-convex and convex functions. This shows SGD attains fast convergence under interpolation \emph{without} explicit knowledge of the Lipschitz constant. Next, we use the above line-search to derive convergence rates of SGD for non-convex functions. 

%% file: SGD-nc.tex
\section{Stochastic Gradient Descent for Non-convex Functions}
\label{sec:sgd-nc}
To prove convergence results in the non-convex case, we additionally require the strong growth condition (SGC)~\cite{vaswani2019fast, schmidt2013fast} to hold. The function $f$ satisfies the SGC with constant $\rho$, if $\E_{i} \norm{\gradi{\x}}\kern-.1em{}^2 \leq \rho \norm{\grad{\x}}\kern-.1em{}^2$ holds for any point $\x$. This implies that if $\grad{\x} = 0$, then $\gradi{\x} = 0$ for \emph{all} $i$. Thus, functions satisfying the SGC necessarily satisfy the interpolation property. The SGC holds for all smooth functions satisfying a PL condition~\cite{vaswani2019fast}. Under the SGC, we show that by upper-bounding the maximum step-size $\eta_{max}$, SGD with Armijo line-search achieves an $O(1/T)$ convergence rate. 
\begin{theorem}[Non-convex]
Assuming (a) the SGC with constant $\rho$ and (b) $L_i$-smoothness of $f_i$'s, SGD with Armijo line-search in Equation~\ref{eq:c-ls} with $c >  1 - \frac{\Lmax}{\rho L}$ and setting $\etamax <  \frac{2}{\rho L}$ achieves the rate:
\begin{align*}
\min_{k = 0, \ldots, T-1} \E \normsq{\grad{\xk}} & \leq \frac{1}{\delta \, T} \, (f(\x_0) - f(\xopt)),
\end{align*}
where \( \delta = \rbr{\etamax + \frac{2(1-c)}{\Lmax}} - \rho \rbr{\etamax - \frac{2(1-c)}{\Lmax} + L \etamax^2} \).
\label{thm:sgd-nc}
\end{theorem}
\vspace*{-1ex}
We prove Theorem~\ref{thm:sgd-nc} in Appendix~\ref{app:nc}. The result requires knowledge of $\rho \; L_\text{max}$ to bound the maximum step-size, which is less practically appealing. It is not immediately clear how to relax this condition and we leave it for future work. In Appendix~\ref{app:add_nc_proofs}, we show that SGD with the Armijo line-search obtains a similar $O(1/T)$ rate with slightly relaxed conditions on $c$ and $\eta_{\text{max}}$ when $(\etak)$ are non-increasing or the Armijo condition holds on a mini-batch which is \textit{independent} of $\gradk{\xk}$. Moreover, in the next section, we show that if the non-convex function satisfies a specific curvature condition, a modified stochastic extra-gradient algorithm can achieve a linear rate under interpolation without additional assumptions or knowledge of the Lipschitz constant.




%% file: SEG.tex
\section{Stochastic Extra-Gradient Method}
\label{sec:seg}
In this section, we use a modified stochastic extra-gradient (SEG) method for convex and non-convex minimization. For finite-sum minimization, stochastic extra-gradient (SEG) has the following update:
\begin{align}
\xkh = \xk - \etak \gradk{\xk} \;,\; \xkk = \xk - \etak \gradk{\xkh} \label{eq:SEG} 
\end{align}
It computes the gradient at an extrapolated point $\xkh$ and uses it in the update from the current iterate $\xk$. Note that using the same sample $i_k$ and step-size $\etak$ for both steps~\cite{gidel2018variational} is important for the subsequent theoretical results. We now describe a ``Lipschitz'' line-search strategy~\cite{khobotov1987modification,iusem1997variant,iusem2017extragradient} in order to automatically set the step-size for SEG. 

\subsection{Lipschitz line-search}
The ``Lipschitz'' line-search has been used by previous work in the deterministic~\cite{khobotov1987modification,iusem1997variant} and the variance reduced settings~\cite{iusem2019variance}. It selects a step-size $\etak$ that satisfies the following condition:
\begin{align}
\norm{\gradk{\xk - \etak \gradk{\xk}} - \gradk{\xk}} & \leq c \; \norm{\gradk{\xk}} 
\label{eq:lip-ls}  
\end{align}
As before, we use backtracking line-search starting from the maximum value of $\eta_{\text{max}}$ to ensure that the chosen step-size satisfies the above condition. If the function $\fk$ is $\Lk$-smooth, the step-size returned by the Lipschitz line-search satisfies $\etak \geq \min\left\{\nicefrac{c}{\Lk}, \eta_{\text{max}}\right\}$. Like the Armijo line-search in Section~\ref{sec:sgd-c}, the Lipschitz line-search does not require knowledge of the Lipschitz constant. Unlike the line-search strategy in the previous sections, checking condition~\eqref{eq:lip-ls}  requires computing the gradient at a prospective extrapolation point. We now prove convergence rates for SEG with Lipschitz line-search for both convex and a special class of non-convex problems. 

\subsection{Convergence rates for minimization}
For the next result, we assume that each function $\fj(\cdot)$ satisfies the restricted secant inequality (RSI) with constant $\mu_i$, implying that for all $\x$, $\langle \gradi{\x}, \x - \x^* \rangle \geq \mu_i \norm{\x - \x^*}\kern-.1em{}^2$. RSI is a weaker condition than strong-convexity. With additional assumptions, RSI is satisfied by important non-convex models such as single hidden-layer neural networks~\cite{li2017convergence, kleinberg2018alternative, soltanolkotabi2018theoretical}, matrix completion~\cite{sun2016guaranteed} and phase retrieval~\cite{chen2015solving}. Under interpolation, we show SEG results in linear convergence for functions satisfying RSI. In particular, we obtain the following guarantee:

\begin{theorem}[Non-convex + RSI]
Assuming (a) interpolation, (b) $L_i$-smoothness, and (c) $\mu_i$-RSI of $f_i$'s, SEG with Lipschitz line-search in Eq.~\ref{eq:lip-ls} with $c = \nicefrac{1}{4}$ and $\eta_{\text{max}} \leq \min_i 1/4 \mu_i$ achieves the rate:
\begin{align*}
\E \left[ \normsq{\x_T - \mathcal{P}_{\mathcal{X}^{*}}[\x_T]} \right] & \leq \max \left\{ \left(1 - \frac{\bar \mu}{4 \; L_{\text{max}}} \right),\left( 1 - \eta_{\text{max}} \; \bar \mu \right) \right\}^{T} \normsq{\x_0 - \mathcal{P}_{\mathcal{X}^{*}}[\x_0]},
\end{align*}
where $\bar{\mu} = \frac{\sum_{i = 1}^{n} \mu_i}{n}$ is the average RSI constant of the finite sum and $\mathcal{X}^{*}$ is the non-empty set of optimal solutions. The operation $\mathcal{P}_{\mathcal{X}^{*}}[\x]$ denotes the projection of $\x$ onto $\mathcal{X}^{*}$. 
\label{thm:seg-min-ls}
\vspace{-1.5ex}
\end{theorem}
See Appendix~\ref{app:seg-min-ls-rsi} for proof. Similar to the result of Theorem~\ref{thm:sc}, the rate depends on the average RSI constant. Note that we do not require explicit knowledge of the Lipschitz constant to achieve the above rate. The constraint on the maximum step-size is mild since the minimum $\mu_i$ is typically small, thus allowing for large step-sizes. Moreover, Theorem~\ref{thm:seg-min-ls} improves upon the $\left(1 - \mu^2/L^2\right)$ rate obtained using constant step-size SGD~\cite{vaswani2019fast, bassily2018exponential}. In Appendix~\ref{app:seg-min-ls-rsi}, we show that the same rate can be attained by SEG with a constant step-size. In Appendix~\ref{app:seg-min-ls}, we show that under interpolation, SEG with Lipschitz line-search also achieves the desired $O(1/T)$ rate for convex functions.

\subsection{Convergence rates for saddle point problems}
In Appendix~\ref{app:seg-sc-minmax}, we use SEG with Lipschitz line-search for a class of saddle point problems of the form $\min_{u \in {U}} \max_{v \in \mathcal{V}} \phi(u,v)$. Here $\mathcal{U}$ and $\mathcal{V}$ are the constraint sets for the variables $u$ and $v$ respectively. In Theorem~\ref{thm:seg-sc-minmax} in Appendix~\ref{app:seg-sc-minmax}, we show that under interpolation, SEG with Lipschitz line-search results in linear convergence for functions $\phi(u,v)$ that are strongly-convex in $u$ and strongly-concave in $v$. The required conditions are satisfied for robust optimization~\cite{wen2014robust} with expressive models capable of interpolating the data. Furthermore, the interpolation property can be used to improve the convergence for a bilinear saddle-point problem~\cite{gidel2018variational,yadav2017stabilizing,mescheder2017numerics,goodfellow2016nips}. In Theorem~\ref{thm:seg-bilin-ls} in Appendix~\ref{sec:seg-bilin-ls}, we show that SEG with Lipschitz line-search results in linear convergence under interpolation. We empirically validate this claim with simple synthetic experiments in Appendix~\ref{app:experiments-games}.

%% file: Practical.tex
\section{Practical Considerations}
\label{sec:practical}
\input{SGD-Armijo-Alg}
In this section, we give heuristics to use larger step-sizes across iterations and discuss ways to use common acceleration schemes with our line-search techniques. 

\subsection{Using larger step-sizes}
\label{sec:eta-tricks}
Recall that our theoretical analysis assumes that the line-search in \emph{each} iteration starts from a global maximum step-size $\eta_\text{max}$. However, in practice this strategy increases the amount of backtracking and consequently the algorithm's runtime. 
A simple alternative is to initialize the line-search in each iteration to the step-size selected in the previous iteration, $\eta_\text{max} = \eta_{k-1}$. Unfortunately, with this strategy the step-size can not increase and convergence is slowed in practice (it takes smaller steps than necessary). 

There are a variety of strategies available that can increase the initial step-size $\eta_\text{max}$ between iterations to improve the practical performance of line-search methods~\cite[Chapter~3]{nocedal2006numerical}. We consider increasing the step-size across iterations by initializing the backtracking at iteration $k$ with $\eta_{k-1} \cdot {\gamma}^{b/n}$, where $b$ is the size of the mini-batch and $\gamma > 1$ is a tunable parameter. These heuristics correspond to the options used in Algorithm \ref{alg:reset_options}. This approach has previously been used in the context of VR-SGD methods~\cite{schmidt2017minimizing, schmidt2015non}, and several related stochastic methods have also appeared in the recent literature~\cite{bollapragada2018progressive,paquette2018stochastic,truong2018backtracking}.

We also consider the Goldstein line-search that uses additional function evaluations to check the curvature condition $\fk \left(\xk - \etak \gradk{\xk} \right) \geq \fk(\xk) -  (1 - c) \cdot \etak \norm{ \gradk{\x_k}}\kern-.1em{}^2$ and increases the step-size if it is not satisfied. Here, $c$ is the constant in Equation~\ref{eq:c-ls}. The resulting method decreases the step-size if the Armijo condition is not satisfied and increases it if the curvature condition does not hold. Algorithm \ref{alg:SGD_Goldstein} in Appendix \ref{app:algorithms} gives pseudo-code for SGD with the Goldstein line-search. 

\subsection{Acceleration}
\label{sec:acceleration}
In practice, augmenting stochastic methods with some form of momentum or acceleration~\cite{polyak1964some,nesterov2007gradient} often results in faster convergence~\cite{sutskever2013importance}. Related work in this context includes algorithms specifically designed to achieve an accelerated rate of convergence in the stochastic setting~\cite{allen2017katyusha,lin2015universal,frostig2015regularizing}. Unlike these works, our experiments considered simple ways of using either Polyak~\cite{polyak1964some} or Nesterov~\cite{nesterov2007gradient} acceleration with the proposed line-search techniques.\footnote{Similar methods have also been explored empirically in prior work~\cite{tseng1998incremental,truong2018backtracking}.} In both cases, similar to adaptive methods using momentum~\cite{sutskever2013importance}, we use SGD with Armijo line-search to determine $\etak$ and then use it directly within the acceleration scheme. When using Polyak momentum, the effective update can be given as: $\xkk = \xk - \etak \gradk{\xk} + \alpha (\xk - \x_{k-1})$, where $\alpha$ is the momentum factor. This update rule has been used with a constant step-size and proven to obtain linear convergence rates on the \emph{generalization error} for quadratic functions under an interpolation condition~\cite{loizou2017momentum, loizou2017linearly}. For Nesterov acceleration, we use the variant for the convex case~\cite{nesterov2007gradient} (which has no additional hyper-parameters) with our line-search. 
The pseudo-code for using these methods with the Armijo line-search is given in Appendix~\ref{app:algorithms}.

%% file: SGD-Armijo-Alg.tex
\begin{figure*}[t!]
    \begin{minipage}[t]{3.1in}
		\begin{algorithm}[H]
			\caption{\texttt{SGD+Armijo}($f$, $\x_0$, $\eta_{\text{max}}$, $b$, $c$, $\beta$, $\gamma$, \texttt{opt})}
			\begin{algorithmic}[1]
			\For{$k = 0, \dots, T$}
			    \State $i_k \gets$ sample mini-batch of size $b$
			    \State $\eta \gets$ \texttt{reset}$(\eta, \eta_{\text{max}}, \gamma, b, k, \mbox{\texttt{opt}}) / \beta $
			    \Repeat
			        \State $\eta \gets \beta \cdot \eta$
			        \State $\tilde \xk \gets \xk - \eta \gradk{\xk}$
			    \Until{$\fk(\tilde \xk) \leq \fk(\xk) -  c \cdot \eta \normsq{ \gradk{\x_k}  }$}
			    \State $\xkk \gets \tilde \xk$ 
			\EndFor
			\State \Return $\xkk$
			\end{algorithmic}
			\label{alg:SGD_Armijo}
		\end{algorithm}
	\end{minipage}
	\hfill
	\begin{minipage}[t]{2.3in}
		\begin{algorithm}[H]
			\caption{\texttt{reset}($\eta$, $\eta_{\text{max}}$, $\gamma$, $b$, $k$, \texttt{opt})}
			\begin{algorithmic}[1]
			\If{k = 1}
			    \State \Return $\eta_{\text{max}}$
			\ElsIf{\texttt{opt} $= 0$}       
			    \State $\eta \gets \eta$
			\ElsIf{\texttt{opt} $= 1$}  
			    \State $\eta \gets \eta_{\text{max}}$
			\ElsIf{\texttt{opt} $= 2$}       
			    \State $\eta \gets \eta \cdot \gamma^{b / n}$
			\EndIf
			\State \Return $\eta$
			\end{algorithmic}
			\label{alg:reset_options}
		\end{algorithm}
	\end{minipage}
	\caption{Algorithm \ref{alg:SGD_Armijo} gives pseudo-code for SGD with Armijo line-search. Algorithm \ref{alg:reset_options} implements several heuristics (by setting \texttt{opt}) for resetting the step-size at each iteration. }
	\label{fig:armijo_with_reset}
\vspace{1ex}
\end{figure*}

%% file: Experiments.tex
\section{Experiments}
\label{sec:experiments}

We describe the experimental setup in Section~\ref{sec:experiments-setup}. In Section~\ref{sec:experiments-synthetic}, we present synthetic experiments to show the benefits of over-parametrization. In Sections~\ref{sec:experiments-kernels} and \ref{sec:experiments-deep}, we showcase the convergence and generalization performance of our methods for kernel experiments and deep networks, respectively. 
\subsection{Experimental setup}
\label{sec:experiments-setup}
We benchmark five configurations of the proposed line-search methods:  SGD with (1) Armijo line-search with resetting the initial step-size (Algorithm~\ref{alg:SGD_Armijo} using option $2$ in Algorithm~\ref{alg:reset_options}), (2) Goldstein line-search (Algorithm~\ref{alg:SGD_Goldstein}),  (3) Polyak momentum (Algorithm~\ref{alg:Polyak_Armijo}), (4) Nesterov acceleration (Algorithm~\ref{alg:Nesterov_Armijo}), and (5) SEG with Lipschitz line-search (Algorithm~\ref{alg:SEG_Armijo}) with option $2$ to reset the step-size. Appendix F gives additional details on our experimental setup and the default hyper-parameters used for the proposed line-search methods. We compare our methods against Adam~\cite{kingma2014adam}, which is the most common adaptive method, and other methods that report better performance than Adam: coin-betting~\cite{orabona2017training}, L4\footnote{L4 applied to momentum SGD (L4 Mom) in \url{https://github.com/iovdin/l4-pytorch} was unstable in our experiments and we omit it from the main paper.}~\cite{rolinek2018l4}, and Adabound~\cite{luo2019adaptive}. We use the default learning rates for the competing methods. Unless stated otherwise, our results are averaged across $5$ independent runs.

\begin{figure}[hbt!]
    \includegraphics[width = \textwidth]{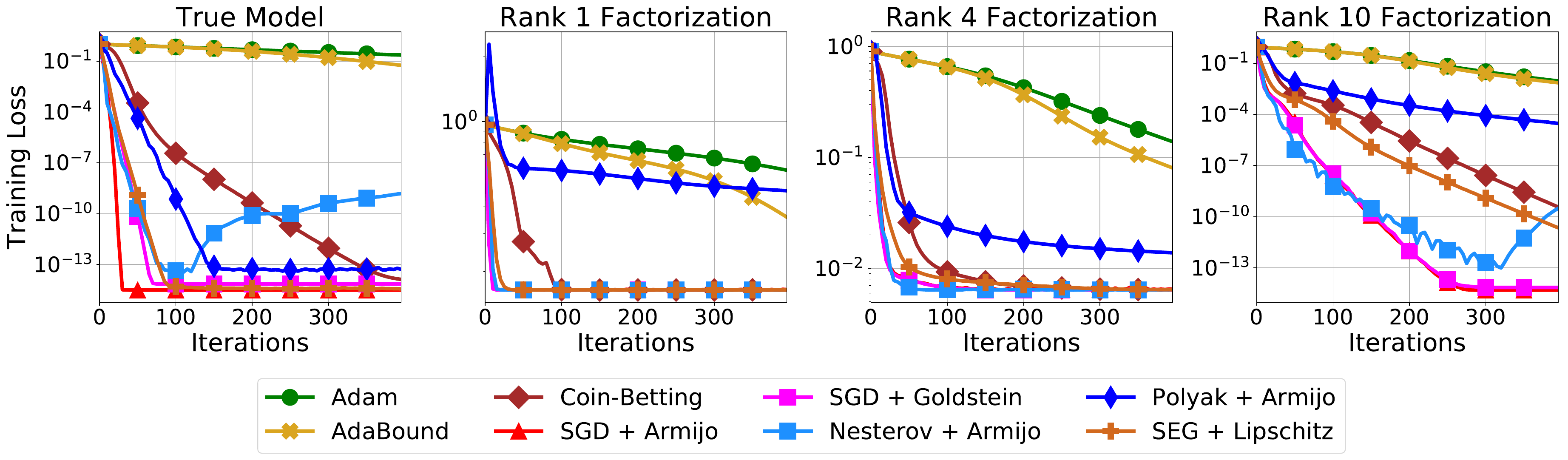}
\caption{Matrix factorization using the true model and rank $1$, $4$, $10$ factorizations. Rank $1$ factorization is under-parametrized, while ranks $4$ and $10$ are over-parametrized. Rank $10$ and the true model satisfy interpolation. }
\label{fig:synthetic}
\end{figure}

\subsection{Synthetic experiment}
\label{sec:experiments-synthetic}
We examine the effect of over-parametrization on convergence rates for the non-convex regression problem: $\min_{W_1, W_2} \E_{x \sim N(0,I)} \norm{W_2 W_1 x - Ax}\kern-.1em{}^2$. This is equivalent to a matrix factorization problem satisfying RSI~\cite{sun2016guaranteed} and has been proposed as a challenging benchmark for gradient descent methods~\cite{rahimi2017reflections}. Following Rolínek et al.~\cite{rolinek2018l4}, we choose $A \in \mathbb{R}^{10 \times 6}$ with condition number $\kappa(A) = 10^{10}$ and generate a fixed dataset of $1000$ samples. Unlike the previous work, we consider stochastic optimization and control the model's expressivity via the rank $k$ of the matrix factors $W_1 \in \mathbb{R}^{k \times 6}$ and $W_2 \in \mathbb{R}^{10 \times k}$. Figure \ref{fig:synthetic} shows plots of training loss (averaged across $20$ runs) for the true data-generating model, and using factors with rank $k \in \{1,4,10\}$. 

We make the following observations: (i) for $k=4$ (where interpolation \emph{does not hold}) the proposed methods converge quicker than other optimizers but all methods reach an artificial optimization floor, (ii) using $k = 10$ yields an over-parametrized model where SGD with both Armijo and Goldstein line-search converge linearly to machine precision, (iii) SEG with Lipschitz line-search obtains fast convergence according to Theorem \ref{thm:seg-min-ls}, and (iv) adaptive-gradient methods stagnate in all cases. These observations validate our theoretical results and show that over-parameterization and line-search can allow for fast, ``painless'' optimization using SGD and SEG.

\begin{figure}[hbt!]
    \includegraphics[width = \textwidth]{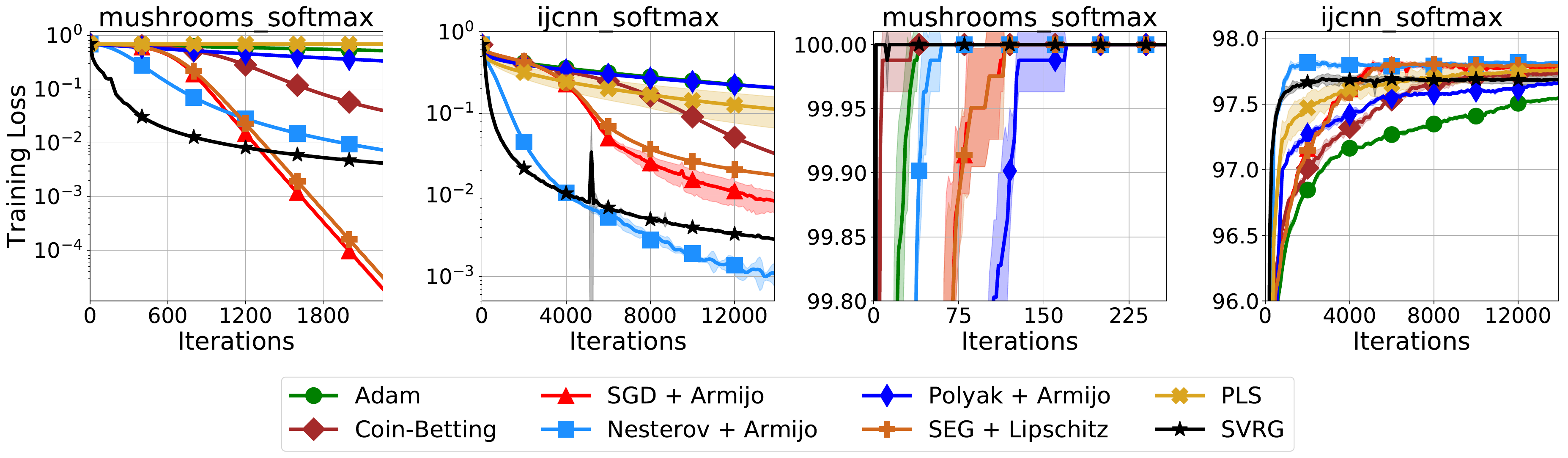}
\caption{Binary classification using a softmax loss and RBF kernels for the mushrooms and ijcnn datasets. Mushrooms is linear separable in kernel-space with the selected kernel bandwidths while ijcnn is \textit{not}. Overall, we observe fast convergence of SGD + Armijo, Nesterov + Armijo, and SEG + Lipschitz for both datasets.}
\label{fig:kernel_results}
\end{figure}

\subsection{Binary classification with kernels}
\label{sec:experiments-kernels}
We consider convex binary classification using RBF kernels without regularization. We experiment with four standard datasets: mushrooms, rcv1, ijcnn, and w8a from LIBSVM~\cite{libsvm}. The mushrooms dataset satisfies the interpolation condition with the selected kernel bandwidths, while ijcnn, rcv1, and w8a do not. For these experiments we also compare against a standard VR method (SVRG)~\cite{johnson2013accelerating} and probabilistic line-search (PLS)~\cite{mahsereci2017probabilistic}.\footnote{PLS is impractical for deep networks since it requires the second moment of the mini-batch gradients and needs GP model inference for every line-search evaluation.} Figure \ref{fig:kernel_results} shows the training loss and test accuracy on mushrooms and ijcnn for the different optimizers with softmax loss. Results for rcv1 and w8a are given in Appendix \ref{app:additional-experiments-kernels}. We make the following observations: (i) SGD + Armijo, Nesterov + Armijo, and SEG + Lipschitz perform the best and are comparable to hand-tuned SVRG. (ii) The proposed line-search methods perform well on ijcnn even though it is not separable in kernel space. This demonstrates some robustness to violations of the interpolation condition.

\begin{figure}[hbt!]
    \centering
    \subfigure{
    \includegraphics[width = .9\textwidth]{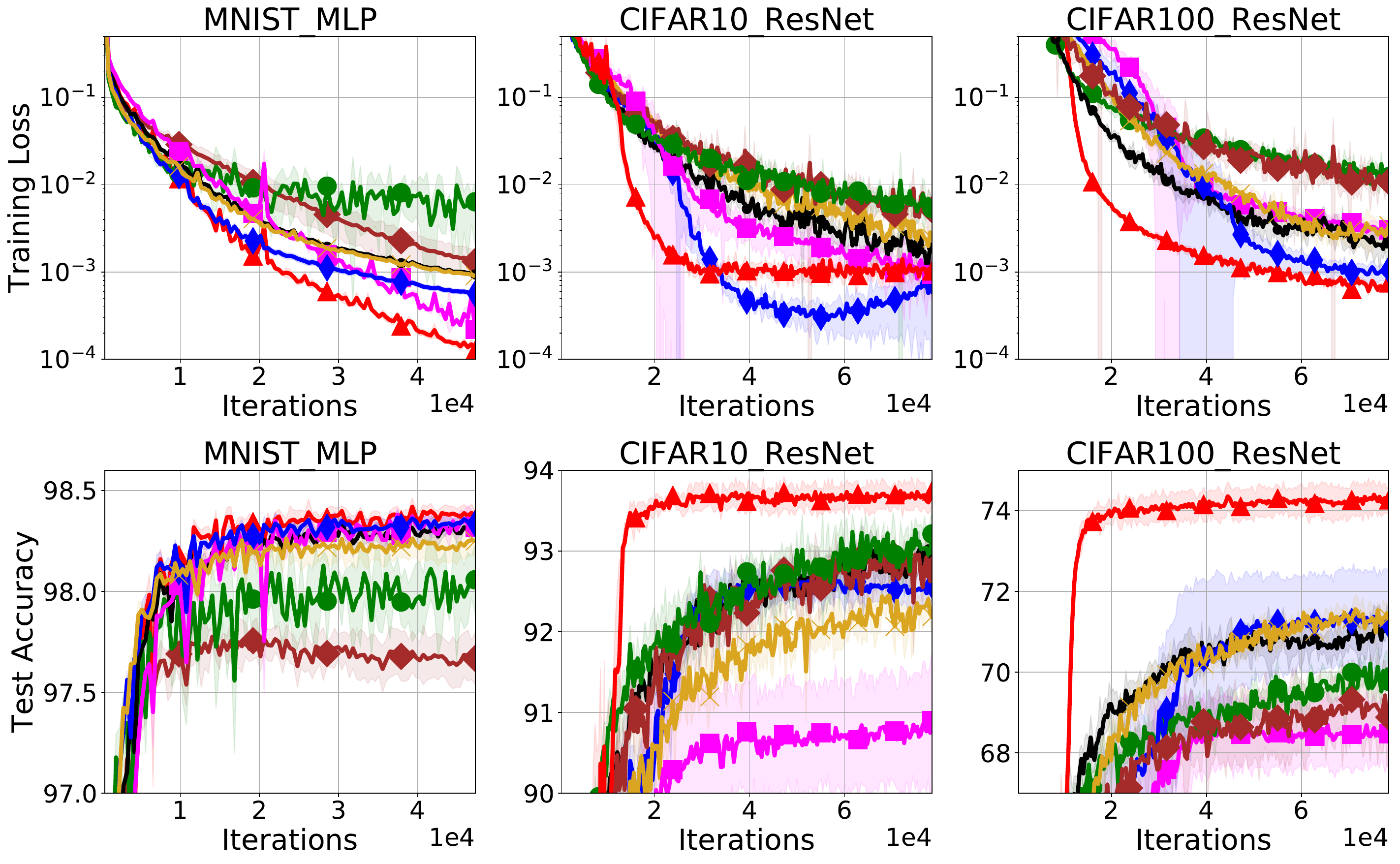}
    }
    \subfigure{
    \includegraphics[width = \textwidth]{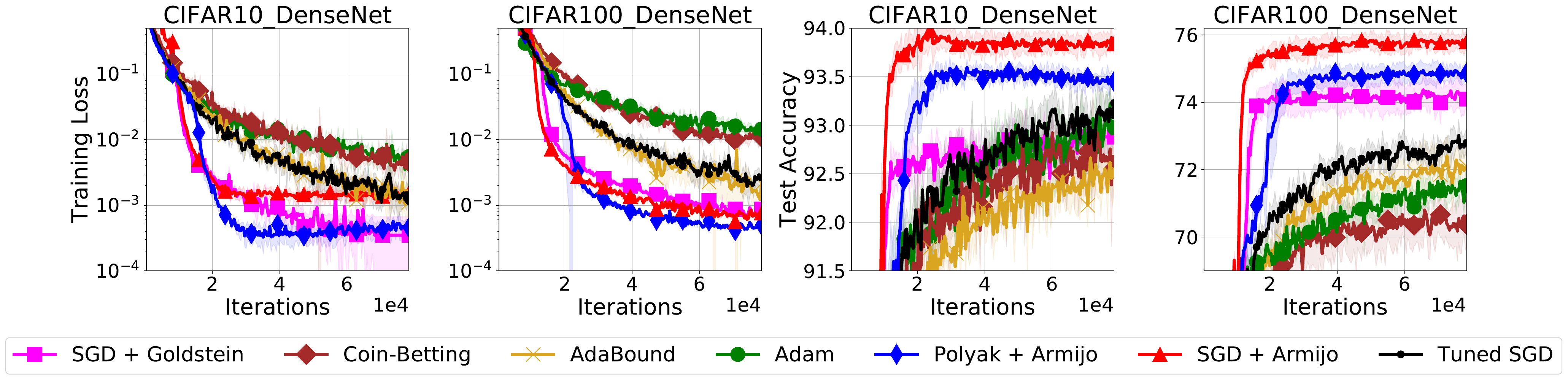}
    }
\caption{Multi-class classification using softmax loss and (top)  an MLP model for MNIST; ResNet model for CIFAR-10 and CIFAR-100 (bottom) DenseNet model for CIFAR-10 and CIFAR-100.}
\label{fig:deep}
\end{figure}

\subsection{Multi-class classification using deep networks}
\label{sec:experiments-deep}
We benchmark the convergence rate and generalization performance of our line-search methods on standard deep learning experiments. We consider non-convex minimization for multi-class classification using deep network models on the MNIST, CIFAR10, and CIFAR100 datasets. Our experimental choices follow the setup in Luo et al.~\cite{luo2019adaptive}. For MNIST, we use a $1$ hidden-layer multi-layer perceptron (MLP) of width $1000$. For CIFAR10 and CIFAR100, we experiment with the standard image-classification architectures: ResNet-34~\cite{he2016deep} and DenseNet-121~\cite{huang2017densely}. We note that promising empirical results for a variety of line-search methods related to those used in our experiments have previously been reported for the CIFAR10 and CIFAR100 datasets for a variety of architectures~\cite{truong2018backtracking}.
We also compare to the best performing constant step-size SGD with the step-size selected by grid search.  

From Figure~\ref{fig:deep}, we observe that: (i) SGD with Armijo line-search consistently leads to the best performance in terms of both the training loss and test accuracy. It also converges to a good solution \emph{much} faster when compared to the other methods. (ii) The performance of SGD with line-search and Polyak momentum is always better than ``tuned'' constant step-size SGD and Adam, whereas that of SGD with Goldstein line-search is competitive across datasets. We omit Nesterov + Armijo as it is unstable and diverges and omit SEG since it resulted in slower convergence and worse performance.

\begin{figure}[hbt!]
    \centering
    \subfigure{
    \includegraphics[width = .23\textwidth]{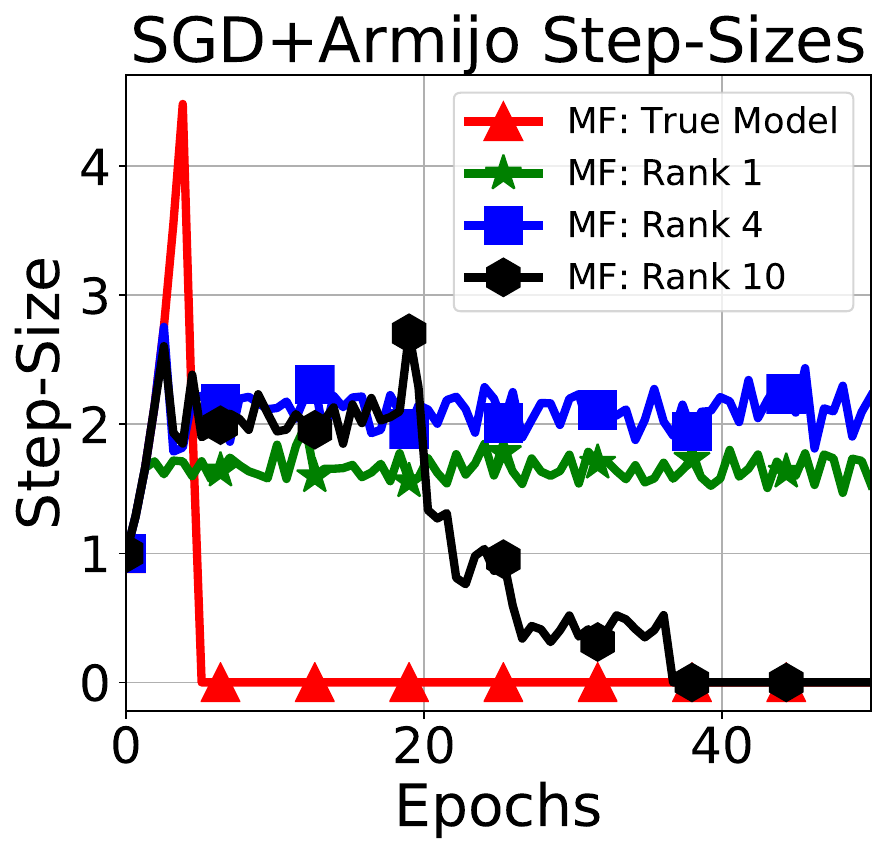}
    }
    \subfigure{
    \includegraphics[width = .23\textwidth]{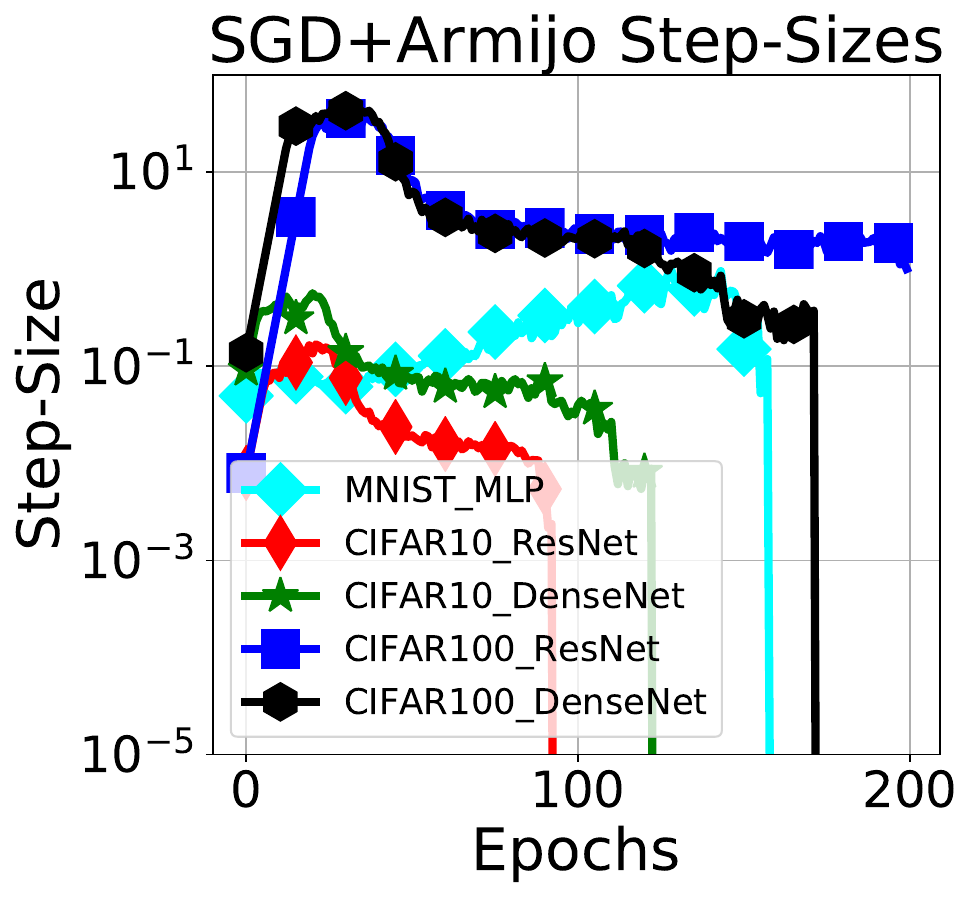}
    }
    \subfigure{
    \includegraphics[width = 0.23\textwidth]{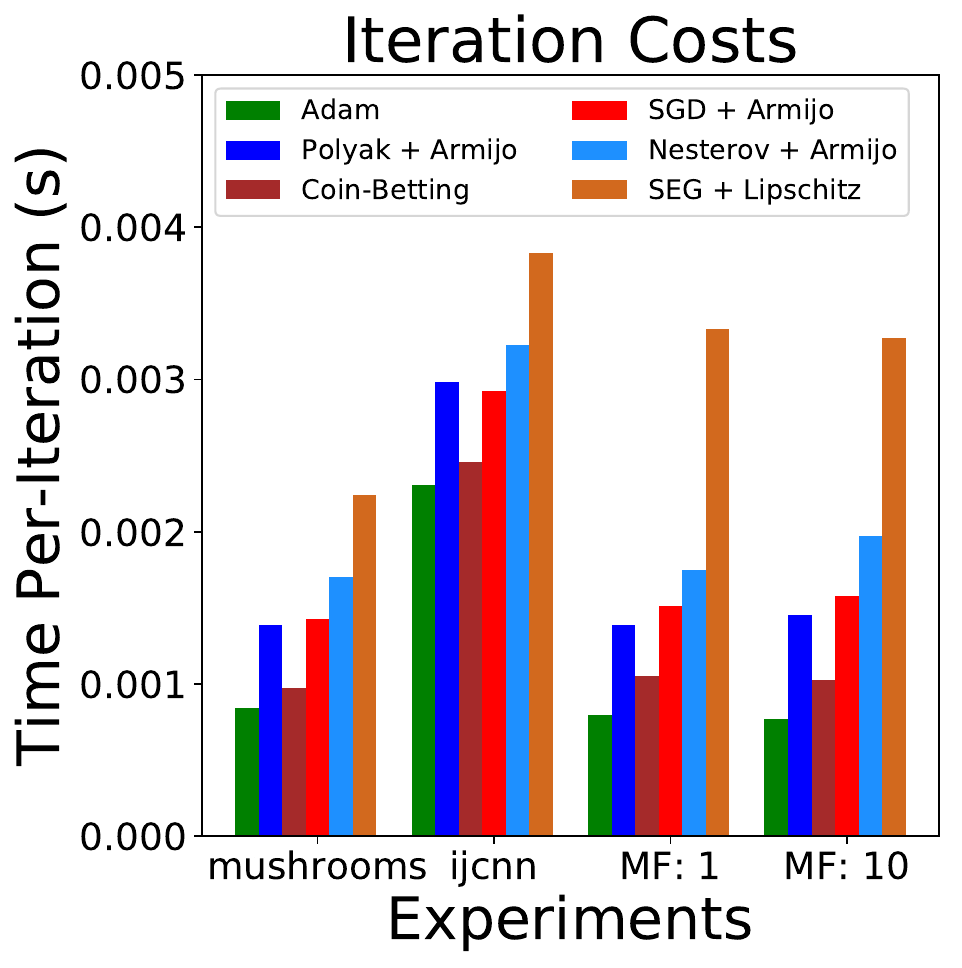}
    }
    \subfigure{
    \includegraphics[width = 0.23\textwidth]{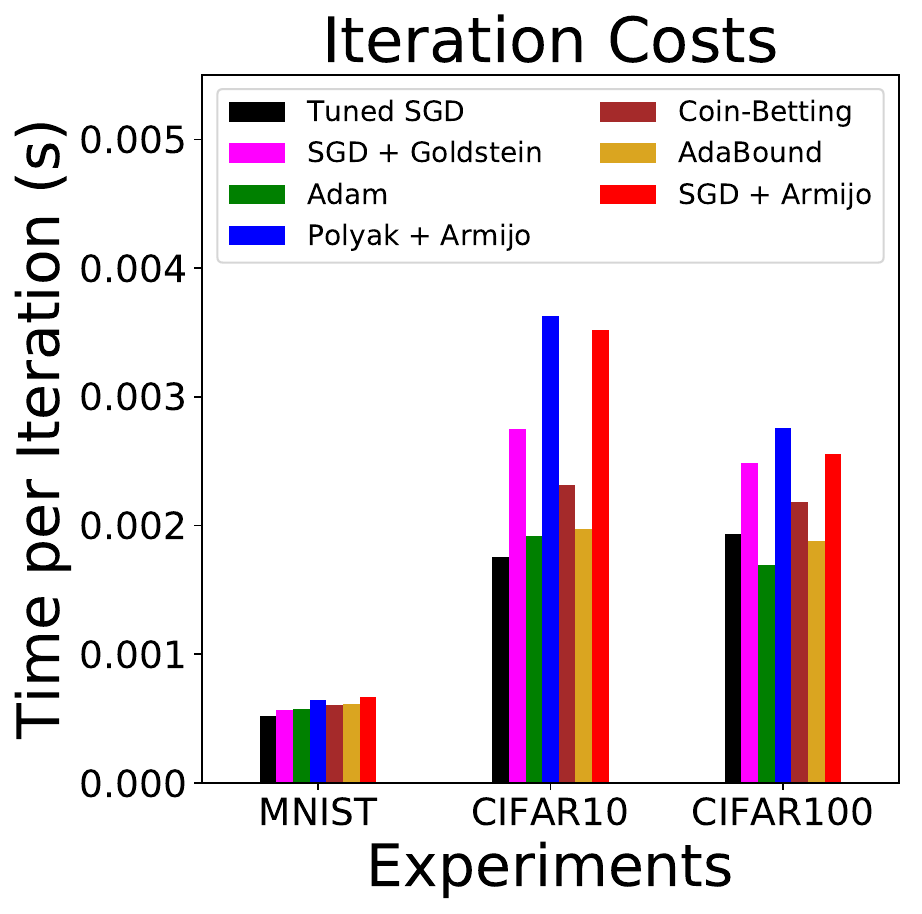}
    }
\caption{(Left) Variation in step-sizes for SGD+Armijo for the matrix factorization problem and classification with deep neural networks. (Right) Average time per iteration.}
\label{fig:time-step}
\vspace{1ex}
\end{figure}

We also verify that our line-search methods do not lead to excessive backtracking and function evaluations. Figure~\ref{fig:time-step}  (right) shows the cost per iteration for the above experiments. Our line-searches methods are only marginally slower than Adam and converge much faster. In practice, we observed SGD+Armijo uses only one additional function evaluation on average. Figure~\ref{fig:time-step}  (left) shows the evolution of step-sizes for SGD+Armijo in our experiments. For deep neural networks, SGD+Armijo automatically finds a step-size schedule resembling cosine-annealing~\cite{loshchilov2016sgdr}. In Appendix~\ref{app:experiments-robustness-computation}, we evaluate and compare the hyper-parameter sensitivity of Adam, constant step-size SGD, and SGD with Armijo line-search on CIFAR10 with ResNet-34. While SGD is sensitive to the choice of the step-size, the performance of SGD with Armijo line-search is robust to the value of $c$ in the ${[0.1,0.5}]$ range. There is virtually no effect of $\eta_\text{max}$, since the correct range of step-sizes is found in early iterations.

%% file: Conclusion.tex
\section{Conclusion}
\label{sec:conclusion}
We showed that under the interpolation condition satisfied by modern over-parametrized models, simple line-search techniques for classic SGD and SEG lead to fast convergence in both theory and practice. For future work, we hope to strengthen our results for non-convex minimization using SGD with line-search and study stochastic momentum techniques under interpolation. More generally, we hope to utilize the rich literature on line-search and trust-region methods to improve stochastic optimization for machine learning. 

%% file: Acknowledgements.tex
\subsubsection*{Acknowledgments} %
\label{par:paragraph_name}
We would like to thank Yifan Sun and Nicolas Le Roux for insightful discussions and Nicolas Loizou and Frederik Kunstner for their help with the proofs. AM is supported by the NSERC CGS M award. IL is funded by the UBC Four-Year Doctoral Fellowships (4YF),  This research was also partially supported by the Canada CIFAR AI Chair Program, the CIFAR LMB Program, by a Google Focused Research award, by an IVADO postdoctoral scholarship (for SV), by a Borealis AI fellowship (for GG), by the Canada Excellence Research Chair in "Data Science for Realtime Decision-making" and by the NSERC Discovery Grants RGPIN-2017-06936 and 2015-06068.

%% file: App-Proofs-common.tex
\section{Proof of Lemma~\ref{lemma:eta-bounds}}
\label{app:lemmas}
\begin{proof}
\begin{align*}
\intertext{From the smoothness of $\fk$ and the update rule, the following inequality holds for all values of $\etak$.} 
\fk(\xkk) & \leq \fk(\xk) - \left(\etak - \frac{\Lk \etak^{2}}{2} \right) \normsq{ \gradk{\xk} } \\
\intertext{The step-size returned by the line-search satisfies Equation~\ref{eq:c-ls}, implying that, }
\fk(\xkk) & \leq \fk(\xk) - c \, \etak \, \normsq{ \gradk{\xk} } 
\intertext{Using the above relations, the step-size returned by the line-search satisfies the following inequality,}
c \; \etak & \geq \left(\etak - \frac{\Lk \etak^{2}}{2} \right) \\
\implies \etak & \geq \frac{2 \; (1-c)}{\Lk}
\intertext{This gives us a lower bound on $\etak$.}
\end{align*}
\begin{align*}
\intertext{Let us now upper-bound $\etak$. Using Equation~\ref{eq:c-ls},}
\implies \etak & \leq \frac{\left[ \fk(\xk) - \fk(\xkk) \right] }{c \normsq{ \gradk{x_k}  } } \\
\etak & \leq \frac{\left[ \fk(\xk) - \fk(\xopt) + \fk(\xopt) - \fk(\xkk) \right] }{c \normsq{ \gradk{\x_k}  } } \\
\intertext{By the interpolation condition, $\fk(\xopt) \leq \fk(\x)$ for all functions $i_k$ and points $\x$, $\implies \fk(\xopt) - \fk(\xkk) \leq 0$}
\implies \etak & \leq \frac{\left[ \fk(\xk) - \fk(\xopt) \right] }{c \normsq{ \gradk{\x_k}  } } \\
\intertext{By definition, $\etak \leq \eta_{\text{max}}$. Furthermore, if we each $\fk(\cdot)$ satisfies the either strong-convexity or the Polyak-Lojasiewicz (PL) inequality ~\cite{polyak1963gradient,karimi2016linear} (which is weaker than strong-convexity and does not require convexity), then,}
\fk(\xk) - \fk(\xopt_{k}) & \leq \frac{1}{2 \muk} \normsq{  \gradk{\xk}  } \\
\implies \fk(\xk) - \fk(\xkk) & \leq \frac{1}{2 \muk} \normsq{  \gradk{\xk}  } \\
\fk(\xk) - \fk(\xopt) & \leq \frac{1}{2 \muk} \normsq{  \gradk{\xk}  } & \tag{Using the interpolation condition} \\
\implies \fk(\xk) - \fk(\xkk) & \leq \frac{1}{2 \muk} \normsq{\gradk{\xk}}  & \tag{Since, $\fk(\xopt) \leq \fk(\xkk)$.} \\
\implies c \cdot \etak & \leq \frac{\normsq{\gradk{\xk}}}{ 2 \muk \normsq{ \gradk{\xk}  } } & \tag{From the above relation on $\etak$.} \\
\implies c \cdot \etak & \leq \frac{1}{2 \muk} \\
\intertext{Thus, the step-size returned by the line-search satisfies the relation $\etak \leq \min\{\frac{1}{2 c \cdot \muk}, \eta_{\text{max}}\}$.}
\intertext{From the above relations,}
\etak & \in \left[\min \left\{\frac{2 \; (1-c)}{\Lk}, \eta_{\text{max}} \right\}, \min \left\{\frac{1}{2 c \cdot \muk}, \eta_{\text{max}} \right\} \right] & \nonumber
\end{align*}
\end{proof}

%% file: App-Proofs-SGD-convex.tex
\section{Proof for Theorem~\ref{thm:sc}}
\label{app:sc}
\begin{proof}
\begin{align*}
\normsq{\xkk - \xopt} & = \normsq{\xk - \etak \gradk{\xk} - \xopt}  \\
\normsq{\xkk - \xopt} & = \normsq{\xk - \xopt} - 2 \eta_{k} \langle \gradk{\xk}, \xk - \xopt \rangle + \eta_{k}^2 \normsq{\gradk{\xk}} \\
\intertext{Using strong-convexity of $\fk(\cdot)$ (and setting $\muk = 0$ if the $f_{i_k}$ is not strongly-convex),}
- \langle \gradk{\xk}, \xk - \xopt \rangle &\leq \fk(\xopt) - \fk(\xk) - \frac{\muk}{2} \normsq{\xk - \xopt} \\
\implies \normsq{\xkk - \xopt} & \leq  \normsq{\xk - \xopt} + 2 \eta_{k} \left[\fk(\xopt) - \fk(\xk) - \frac{\muk}{2} \normsq{\xk - \xopt} \right] + \eta_{k}^2 \normsq{\gradk{\xk}} \\
& = \normsq{\xk - \xopt} + 2 \eta_{k} \left[\fk(\xopt) - \fk(\xk) \right] -\muk \etak\normsq{\xk - \xopt}  + \eta_{k}^2 \normsq{\gradk{\xk}} \\
\implies \normsq{\xkk - \xopt} & \leq \left(1 - \muk \etak \right)\normsq{\xk - \xopt} + 2 \eta_{k} \left[\fk(\xopt) - \fk(\xk) \right] + \eta_{k}^2 \normsq{\gradk{\xk}} \\
\intertext{Using Equation~\ref{eq:c-ls},}
\eta_{k}^2 \normsq{\gradk{\xk}} &\leq \frac{\etak}{c} \left[ \fk(\xk) - \fk(\xkk) \right] \\
\implies \normsq{\xkk - \xopt} & \leq \left(1 - \muk \etak \right)\normsq{\xk - \xopt} + 2 \eta_{k} \left[\fk(\xopt) - \fk(\xk) \right] +  \frac{\etak}{c} \left[ \fk(\xk) - \fk(\xkk) \right] \\
\intertext{The interpolation condition implies that $\xopt$ is the minimum for all functions $\fj$, implying that for all $i$, $\fj(\xopt) \leq \fj(\xkk)$.}
\normsq{\xkk - \xopt} & \leq \left(1 - \muk \etak \right)\normsq{\xk - \xopt} + 2 \eta_{k} \left[\fk(\xopt) - \fk(\xk) \right] +  \frac{\etak}{c} \left[ \fk(\xk) - \fk(\xopt) \right] \\
& = \left(1 - \muk \etak \right)\normsq{\xk - \xopt} + \left( 2 \eta_{k} - \frac{\etak}{c} \right) \, \left[\fk(\xopt) - \fk(\xk) \right]  \\
\intertext{The term $\left[\fk(\xopt) - \fk(\xk) \right]$ is negative. Let $c \geq \frac{1}{2} \implies \left( 2 \eta_{k} - \frac{\etak}{c} \right) \geq 0$ for all $\etak$.}
\implies \normsq{\xkk - \xopt} & \leq \left(1 - \muk \etak \right)\normsq{\xk - \xopt} \\
\intertext{Taking expectation wrt to $i_k$,}
\implies \E \left[\normsq{\xkk - \xopt} \right] & \leq \E_{ik} \left[ \left(1 - \muk \etak \right)\normsq{\xk - \xopt} \right] \\
& = \left(1 - \E_{ik} \left[ \muk \etak \right] \right) \normsq{\xk - \xopt}  \\
& \leq \left(1 - \E_{ik} \left[ \muk \; \min \left\{\frac{2 \, (1 - c)}{\Lk}, \eta_{\text{max}} \right\} \right] \right) \normsq{\xk - \xopt}  & \tag{Using Equation~\ref{eq:eta-bounds}} \\
\intertext{Setting $c = 1/2$,}
\implies \E \left[\normsq{\xkk - \xopt} \right] & \leq \left(1 - \E_{ik} \left[ \muk \; \min \left\{\frac{1}{\Lk}, \eta_{\text{max}} \right\} \right] \right) \normsq{\xk - \xopt}
\intertext{We consider the following two cases: $\eta_{\text{max}} < 1 / L_{\text{max}}$ and $\eta_{\text{max}} \geq 1 / L_{\text{max}}$. When $\eta_{\text{max}} < 1 / L_{\text{max}}$, we have $\eta_{\text{max}} < 1 / \Lk$ and, }
\E \left[\normsq{\xkk - \xopt} \right] & \leq \left(1 - \E_{ik} \left[ \muk \; \eta_{\text{max}} \right] \right) \normsq{\xk - \xopt}\\
&= \left(1 - \E_{ik} \left[ \muk \right] \; \eta_{\text{max}} \right) \normsq{\xk - \xopt} = \left(1 - \bar \mu \; \eta_{\text{max}} \right) \normsq{\xk - \xopt}\\
\intertext{By recursion through iterations $k = 1$ to $T$,}
\E \left[ \normsq{\x_{T} - \xopt} \right] & \leq \left(1 - \bar \mu \; \eta_{\text{max}} \right)^{T} \normsq{\x_{0} - \xopt}. \\
\intertext{When $\eta_{\text{max}} \geq 1 / L_{\text{max}}$, we use $\min \left\{\frac{1}{\Lk}, \eta_{\text{max}} \right\} \geq \min \left\{\frac{1}{L_{\text{max}}}, \eta_{\text{max}} \right\}$ to obtain }
\E \left[\normsq{\xkk - \xopt} \right] & \leq \left(1 - \E_{ik} \left[ \muk \; \min \left\{\frac{1}{L_{\text{max}}}, \eta_{\text{max}} \right\} \right] \right) \normsq{\xk - \xopt}\\
& = \left(1 - \E_{ik} \left[ \muk \; \frac{1}{L_{\text{max}}} \right] \right) \normsq{\xk - \xopt}\\
& = \left(1 - \frac{\E_{ik} \left[ \muk \right]}{L_{\text{max}}} \right) \normsq{\xk - \xopt} = \left(1 - \frac{\bar \mu}{L_{\text{max}}} \right) \normsq{\xk - \xopt}\\
\intertext{By recursion through iterations $k = 1$ to $T$,}
\E \left[ \normsq{\x_{T} - \xopt} \right] & \leq \left(1 - \frac{\bar \mu}{L_{\text{max}}} \right)^{T} \normsq{\x_{0} - \xopt}.
\intertext{Putting the two cases together,}
\E \left[ \normsq{\x_{T} - \xopt} \right] & \leq \max \left\{ 
\left(1 - \frac{\bar \mu}{L_{\text{max}}} \right), \left(1 - \bar \mu \; \eta_{\text{max}} \right) \right\}^{T} \normsq{\x_{0} - \xopt}
\end{align*}
\end{proof}

\section{Proof for Theorem~\ref{thm:c}}
\label{app:c}
\begin{proof}
\begin{align*}
\normsq{\xkk - \xopt} & = \normsq{\xk - \etak \gradk{\xk} - \xopt}  \\
\normsq{\xkk - \xopt} & = \normsq{\xk - \xopt} - 2 \eta_{k} \langle \gradk{\xk}, \xk - \xopt \rangle + \eta_{k}^2 \normsq{\gradk{\xk}} \\
2 \eta_{k} \langle \gradk{\xk}, \xk - \xopt \rangle & = \normsq{\xk - \xopt} - \normsq{\xkk - \xopt} + \eta_{k}^2 \normsq{\gradk{\xk}} \\
\langle \gradk{\xk}, \xk - \xopt \rangle & = \frac{1}{2 \etak} \left[ \normsq{\xk - \xopt} - \normsq{\xkk - \xopt} \right] + \frac{\eta_{k}}{2} \normsq{\gradk{\xk}} \\
& \leq \frac{1}{2 \etak} \left[ \normsq{\xk - \xopt} - \normsq{\xkk - \xopt} \right] + \frac{\fk(\xk) - \fk(\xkk)}{2 c} & \tag{Using Equation~\ref{eq:c-ls}} \\
\intertext{The interpolation condition implies that $\xopt$ is the minimum for all functions $\fj$, implying that for all $i$, $\fj(\xopt) \leq \fj(\xkk)$.}
\implies \langle \gradk{\xk}, \xk - \xopt \rangle & \leq \frac{1}{2 \etak} \left[ \normsq{\xk - \xopt} - \normsq{\xkk - \xopt} \right] + \frac{\fk(\xk) - \fk(\xopt)}{2 c} \\
\intertext{Taking expectation wrt $i_k$,}
\E \left[ \langle \gradk{\xk}, \xk - \xopt \rangle \right] & \leq \E \left[ \frac{1}{2 \etak} \left[ \normsq{\xk - \xopt} - \normsq{\xkk - \xopt} \right] \right]+ \E \left[\frac{\fk(\xk) - \fk(\xopt)}{2 c} \right] \\
& = \E \left[ \frac{1}{2 \etak} \left[ \normsq{\xk - \xopt} - \normsq{\xkk - \xopt} \right] \right] + \left[\frac{f(\xk) - f(\xopt)}{2 c} \right] \\
\implies \langle \E \left[ \gradk{\xk} \right], \xk - \xopt \rangle & \leq  \E \left[ \frac{1}{2 \etak} \left[ \normsq{\xk - \xopt} - \normsq{\xkk - \xopt} \right] \right] + \left[\frac{f(\xk) - f(\xopt)}{2 c} \right] \\
\implies \langle \grad{\xk}, \xk - \xopt \rangle & \leq  \E \left[ \frac{1}{2 \etak} \left[ \normsq{\xk - \xopt} - \normsq{\xkk - \xopt} \right] \right] + \left[\frac{f(\xk) - f(\xopt)}{2 c} \right] \\
\intertext{By convexity,}
f(\xk) - f(\xopt) &\leq  \langle \grad{\xk}, \xk - \xopt \rangle \\
\implies f(\xk) - f(\xopt) & \leq \E \left[ \frac{1}{2 \etak} \left[ \normsq{\xk - \xopt} - \normsq{\xkk - \xopt} \right] \right] + \left[\frac{f(\xk) - f(\xopt)}{2 c} \right] \\
\intertext{If $1 - \frac{1}{2c} \geq 0$ $\implies$ if $c \geq \frac{1}{2}$, then}
\implies f(\xk) - f(\xopt) & \leq \E \left[ \frac{c}{(2c - 1) \etak} \left[ \normsq{\xk - \xopt} - \normsq{\xkk - \xopt} \right] \right] \\
\intertext{Taking expectation and summing from $k = 0$ to $k = T-1$}
\implies \E \left[ \sum_{k = 0}^{T-1} \left[ f(\xk) - f(\xopt) \right] \right] & \leq \E \left[ \sum_{k = 0}^{T-1}  \frac{c}{(2c - 1) \etak} \left[ \normsq{\xk - \xopt} - \normsq{\xkk - \xopt} \right] \right] \\
\intertext{By Jensen's inequality,}
\E \left[ f(\mxt) - f(\xopt) \right] & \leq \E \left[ \sum_{k = 0}^{T-1} \left[ \frac{f(\xk) - f(\xopt)}{T} \right] \right] \\
\implies \E \left[ f(\mxt) - f(\xopt) \right]  & \leq \frac{1}{T} \E \left[ \sum_{k = 0}^{T-1} \frac{c}{(2c - 1) \etak} \left[ \normsq{\xk - \xopt} - \normsq{\xkk - \xopt} \right] \right] \\
\intertext{If $\Delta_{k} = \normsq{\xk - \xopt}$, then}
\E \left[ f(\mxt) - f(\xopt) \right]  & \leq \frac{c}{T \; (2c - 1)} \E \left[ \sum_{k = 0}^{T-1} \frac{1}{\etak} \left[ \Delta_{k} - \Delta_{k+1} \right] \right] \\
\intertext{Using Equation~\ref{eq:eta-bounds},}
\frac{1}{\etak} & \leq \max\left\{\frac{\Lk}{2 \; (1-c)}, \frac{1}{\eta_{\text{max}}} \right\} \leq \max\left\{\frac{L_{\text{max}}}{2 \; (1-c)}, \frac{1}{\eta_{\text{max}}} \right\} \\
\implies \E \left[ f(\mxt) - f(\xopt) \right]  & \leq \frac{c \cdot \max\left\{\frac{L_{\text{max}}}{2 \; (1-c)}, \frac{1}{\eta_{\text{max}}} \right\}}{(2c - 1) \; T} \E \sum_{k = 0}^{T-1} \left[ \Delta_{k} - \Delta_{k+1} \right] \\
& = \frac{c \cdot \max\left\{\frac{L_{\text{max}}}{2 \; (1-c)}, \frac{1}{\eta_{\text{max}}} \right\}}{(2c - 1) \; T} \E \left[ \Delta_{0} - \Delta_{T} \right] \\
\E \left[ f(\mxt) - f(\xopt) \right] & \leq \frac{c \cdot \max\left\{\frac{L_{\text{max}}}{2 \; (1-c)}, \frac{1}{\eta_{\text{max}}} \right\}}{(2c - 1) \; T} \normsq{\x_{0} - \xopt}
\end{align*}
\end{proof}

%% file: App-Proofs-SGD-nonconvex-correct.tex
\section{Proof for Theorem~\ref{thm:sgd-nc}}
\label{app:nc}
\begin{proof}
Firstly, note that for any vectors \( a,b \in \R^d, \)
\begin{align}
    \norm{a - b}^2 &= \norm{a}^2 + \norm{b}^2 - 2\abr{a,b}\nonumber \\
    \implies - \abr{a, b} &= \frac{1}{2} \rbr{\norm{a - b}^2 - \norm{a}^2 - \norm{b}^2}.~\label{eq:quadratic-expansion}
\end{align}

Let \( \Delta_k = f(\xkk) - f(\xk) \). Starting from \( L \)-smoothness of \( f \):
\begin{align*}
    \Delta_k &\leq \abr{\grad{\xk}, \xkk - \xk} + \frac{L}{2}\norm{\xkk - \xk}^2\\
                     &= - \etak \abr{\grad{\xk}, \gradk{\xk}} + \frac{L \etak^2}{2}\norm{\gradk{\xk}}^2.
                     \intertext{Using \autoref{eq:quadratic-expansion} on \( - \abr{\grad{\xk}, \gradk{\xk}} \), }
    \implies \Delta_k &\leq \frac{\etak}{2}\rbr{\norm{\grad{\xk} - \gradk{\xk}}^2 - \norm{\grad{\xk}}^2 - \norm{\gradk{\xk}}^2} + \frac{L \etak^2}{2}\norm{\gradk{\xk}}^2\\
    \implies 2\Delta_k &\leq \etak \norm{\grad{\xk} - \gradk{\xk}}^2 - \etak\rbr{\norm{\grad{\xk}}^2 + \norm{\gradk{\xk}}^2} + L \etak^2\norm{\gradk{\xk}}^2.
    \intertext{Let \( \Lmax = \max_{i} L_i \). Then, Lemma~\ref{lemma:eta-bounds} guarantees \( \etamin = \min\cbr{\frac{2(1-c)}{\Lmax}, \etamax} \leq \min\cbr{\frac{2(1-c)}{\Lk}, \etamax} \leq \etak \leq \etamax \). Using this and taking expectations with respect to \( \gradk{\xk} \),}
    2 \Delta_k &\leq \etamax \norm{\grad{\xk} - \gradk{\xk}}^2 - \etamin \rbr{\norm{\grad{\xk}}^2 + \norm{\gradk{\xk}}^2} \\ &\hspace{1em} + L \etamax^2\norm{\gradk{\xk}}^2\\
    \implies 2 \E\sbr{\Delta_k} &\leq \etamax \E\sbr{\norm{\grad{\xk} - \gradk{\xk}}^2} - \etamin \E\sbr{\norm{\grad{\xk}}^2 + \norm{\gradk{\xk}}^2} \\ &\hspace{1em} + L \etamax^2 \E\sbr{\norm{\gradk{\xk}}^2}\\
                            &= \etamax \E\sbr{\norm{\gradk{\xk}}^2} - \etamax \norm{\grad{\xk}}^2 - \etamin \E\sbr{\norm{\grad{\xk}}^2 + \norm{\gradk{\xk}}^2} \\ &\hspace{1em} + L \etamax^2 \E\sbr{\norm{\gradk{\xk}}^2}\\
                            \intertext{Collecting terms and applying the strong growth condition,}
    2 \E\sbr{\Delta_k} &\leq \rbr{\etamax - \etamin + L \etamax^2} \E\sbr{\norm{\gradk{\xk}}^2} - \rbr{\etamax + \etamin}\norm{\grad{\xk}}^2\\
                           &\leq \rho \rbr{\etamax - \etamin + L \etamax^2} \norm{\grad{\xk}}^2 - \rbr{\etamax + \etamin}\norm{\grad{\xk}}^2\\
                           &= - \rbr{\rbr{\etamax + \etamin} - \rho \rbr{\etamax - \etamin + L \etamax^2}}\norm{\grad{\xk}}^2.
\end{align*}
\begin{align*}
    \intertext{Assuming \( \delta = \rbr{\etamax + \etamin} - \rho \rbr{\etamax - \etamin + L \etamax^2} > 0 \),}
    \implies \norm{\grad{\xk}}^2 &\leq \frac{2}{\delta} \E\sbr{-\Delta_k}.
    \intertext{Taking expectations and summing from \( k = 0 \) to \( K - 1 \),}
    \implies \frac{1}{K} \sum_{k=0}^{K-1} \E\sbr{\norm{\grad{\xk}}^2} &\leq \frac{2}{\delta \, K} \sum_{k=0}^{K-1} \E \sbr{- \Delta_k}\\
                                                          &= \frac{2}{\delta \, K} \sum_{k=0}^{K-1} \E \sbr{f(\xk) - f(\xkk)}\\
                                                          &= \frac{2}{\delta \, K} \E\sbr{f(\x_0) - f(\xkk)}\\
                                                          &\leq \frac{2}{\delta \, K} \rbr{f(\x_0) - f(\xopt)}\\
    \implies \min_{k \in [T-1]} \E\sbr{\norm{\grad{\xk}}^2} &\leq \frac{2}{\delta \, T} \rbr{f(\x_0) - f(\xopt)}.
\end{align*}

\noindent It remains to show that \( \delta > 0  \) holds. Our analysis proceeds in cases.\\

\noindent \textbf{Case 1}: \( \etamax \leq \frac{2(1-c)}{\Lmax} \). Then \( \etamin = \etamax \) and
\begin{align*}
    \delta &= \rbr{\etamax + \etamax} - \rho \rbr{\etamax - \etamax + L \etamax^2} \\
    &= 2 \etamax - \rho L \etamax^2 > 0 \\
    \implies \etamax &< \frac{2}{\rho L}.
\end{align*}
\textbf{Case 2}: \( \etamax > \frac{2(1-c)}{\Lmax} \). Then \( \etamin = \frac{2(1-c)}{\Lmax} \) and
\begin{align*}
    \delta =  \rbr{\etamax + \frac{2(1-c)}{\Lmax}} - \rho \rbr{\etamax - \frac{2(1-c)}{\Lmax} + L \etamax^2}.
    \intertext{This is a concave quadratic in \( \etamax \) and is strictly positive when}
    \etamax \in \rbr{0, \frac{\rbr{1 - \rho} + \sqrt{{(\rho - 1)}^2 + \frac{\sbr{8 \rho (1 + \rho) L (1-c)}}{\Lmax}}}{2 L \rho}}.
\end{align*}
To avoid contradiction with the case assumption \( \frac{2(1-c)}{\Lmax} < \etamax \), we require
\begin{align*}
    \frac{\rbr{1 - \rho} + \sqrt{{(\rho - 1)}^2 + \frac{\sbr{8 \rho (1 + \rho) L (1-c)}}{\Lmax}}}{2 L \rho} &> \frac{2(1-c)}{\Lmax}\\
    \implies \frac{8 \rho (1 + \rho) L (1-c)}{\Lmax} &> \rbr{\frac{4L \rho}{\Lmax} + (\rho - 1)}^2 - \rbr{\rho - 1}^2\\
                                                           &= \frac{16L^2 \rho^2 (1-c)^2}{\Lmax^2} + \frac{8 L \rho (\rho - 1) (1-c)}{\Lmax}\\
    \implies \frac{\Lmax}{\rho L} &> (1-c)\\
    \implies c &> 1 - \frac{\Lmax}{\rho L}.
\end{align*}
The line-search requires \( c \in \rbr{0,1} \).
Noting that \( \rho \geq 1 \) by definition, we have \( \frac{\Lmax}{\rho L} > 0 \) as long as \( L, \Lmax > 0 \).
The Lipschitz constants are strictly positive when \( f \) is bounded-below and non-zero.
We obtain the non-empty constraint set
\[ c \in \rbr{1 - \frac{\Lmax}{\rho L}, 1}. \]
Substituting the maximum value for \( c \) into the upper-bound on \( \etamax \) yields a similar requirement,
\[ \etamax \in \rbr{0, \frac{2}{\rho L}}. \]
This completes the second case.\\

\noindent Putting the two cases together gives the final constraints on \( c \) and \( \etamax \) as
\[ c \geq 1 - \frac{\Lmax}{\rho L} \quad \quad \etamax < \frac{2}{\rho L}. \]
We note that the upper and lower bounds on \( \etak \) are consistent since
\[ \etamin = \min\cbr{\frac{2(1-c)}{\Lmax}, \etamax} < \min\cbr{\frac{2 \Lmax}{\rho L \Lmax}, \etamax} = \max\cbr{\frac{2}{\rho L}, \etamax} \leq \frac{2}{\rho L}, \]
where the last inequality follows from the bound on \( \etamax \).
In particular, taking \( c \rightarrow 1 \) and \( \etamax \rightarrow \frac{2}{\rho L} \) yields an adaptive step-size \(\etak \in (0,\frac{2}{\rho L}).  \)

\end{proof}

%% file: App-Proofs-SEG.tex
\section{Proofs for SEG}
\label{app:seg}
\subsection{Common lemmas}
\label{app:seg-common}
We denote $\normsq{u - v}$ as $\Delta(u,v) = \Delta(v,u)$. We first prove the following lemma that will be useful in the subsequent analysis.
\begin{lemma}
For any set of vectors $a,b,c,d$, if $a = b + c$, then,
\begin{align*}
    \Delta(a,d) & = \Delta(b,d) - \Delta(a,b) + 2 \langle c, a - d \rangle
\end{align*}
\label{lemma:gen}
\end{lemma}
\begin{proof}
\begin{align*}
\Delta(a,d) & = \normsq{a - d} = \normsq{b + c - d} \\
& = \normsq{b - d} + 2 \langle c, b - d \rangle + \normsq{c} \\
\intertext{Since $c = a - b$,}
\Delta(a,d) & = \normsq{b - d} + 2 \langle a - b, b - d \rangle + \normsq{a - b} \\
& = \normsq{b - d} + 2 \langle a - b, b - a + a - d \rangle + \normsq{a - b} \\
& = \normsq{b - d} + 2 \langle a - b, b - a \rangle + 2 \langle a - b, a - d \rangle + \normsq{a - b} \\
& = \normsq{b - d} - 2 \normsq{a - b} + 2 \langle a - b, a - d \rangle + \normsq{a - b} \\
& = \normsq{b - d} - \normsq{a - b} + 2 \langle c, a - d \rangle \\
\Delta(a,d) & = \Delta(b,d) - \Delta(a,b) + 2 \langle c, a - d \rangle.
\end{align*}
\end{proof}
\subsection{Proof for Theorem~\ref{thm:seg-min-ls}}
\label{app:seg-min-ls-rsi}
We start from Lemma \ref{lemma:gen} with $a = \xkk = \xk - \etak \gradk{\xkh}$ and $d = \xopt$:
\begin{align*}
\Delta(\xkk,\xopt) & = \Delta(\xk, \xopt) - \Delta(\xkk, \xk) - 2 \etak \left[\langle \gradk{\xkh}, \xkk - \xopt \rangle \right].\\
& = \Delta(\xk, \xopt) - \etak^{2} \normsq{\gradk{\xkh}} - 2 \etak \left[\langle \gradk{\xkh}, \xkk - \xopt \rangle \right].\\
\intertext{Using $\xkk = \xkh + \etak \gradk{\xk} - \etak \gradk{\xkh}$ and completing the square,}
\Delta(\xkk,\xopt) & = \Delta(\xk, \xopt) - \etak^{2} \normsq{\gradk{\xkh}} - 2 \etak \left[\langle \gradk{\xkh}, \xkh + \etak \gradk{\xk} - \etak \gradk{\xkh} - \xopt \rangle \right]\\
& = \Delta(\xk, \xopt) + \etak^{2} \normsq{\gradk{\xkh}} - 2 \etak \left[\langle \gradk{\xkh}, \xkh + \etak \gradk{\xk} - \xopt \rangle \right]\\
& = \Delta(\xk, \xopt) + \etak^{2} \normsq{\gradk{\xkh} - \gradk{\xk}} - \etak^2 \normsq{\gradk{\xk}} - 2 \etak \left[\langle \gradk{\xkh}, \xkh - \xopt \rangle \right]\\
\intertext{Noting $\Delta(\xkh, \xk) = \etak^2 \normsq{\gradk{\xk}}$ gives}
\Delta(\xkk,\xopt) & = \Delta(\xk, \xopt) - \Delta(\xkh, \xk) + \etak^{2} \normsq{\gradk{\xkh} - \gradk{\xk}} - 2 \etak \left[\langle \gradk{\xkh}, \xkh - \xopt \rangle \right] \\
\implies 2 \etak \left[\langle \gradk{\xkh}, \xkh - \xopt \rangle \right]
& = \Delta(\xk, \xopt) - \Delta(\xkh, \xk) + \etak^{2} \normsq{\gradk{\xkh} - \gradk{\xk}} - \Delta(\xkk,\xopt) \linenumber \label{eq:SEG_intermediate_inequality}.\\
\intertext{By RSI, which states that for all $\x$, $\langle \gradi{\x}, \x - \xopt \rangle \geq \mu_i \normsq{\xopt - \x}$, we have}
\langle \gradk{\xkh}, \xkh - \xopt \rangle & \geq \muk \Delta(\xkh,\xopt) \\
\intertext{By Young's inequality,}
\Delta(\xk,\xopt) & \leq 2 \Delta(\xk,\xkh) + 2 \Delta(\xkh,\xopt) \\
\implies 2 \Delta(\xkh,\xopt) & \geq \Delta(\xk,\xopt) - 2 \Delta(\xk,\xkh) \\
\implies \langle 2 \etak \gradk{\xkh}, \xkh - \xopt \rangle & \geq \muk \etak \left[ \Delta(\xk,\xopt) - 2 \Delta(\xk,\xkh) \right] \\
\intertext{Rearranging Equation \eqref{eq:SEG_intermediate_inequality},}
\Delta(\xkk,\xopt) &= \Delta(\xk, \xopt) - \Delta(\xkh, \xk) + \etak^{2} \normsq{\gradk{\xkh} - \gradk{\xk}}- 2 \etak \left[\langle \gradk{\xkh}, \xkh - \xopt \rangle \right] \\
\implies \Delta(\xkk,\xopt) & \leq \Delta(\xk, \xopt) - \Delta(\xkh, \xk) + \etak^{2} \normsq{\gradk{\xkh} - \gradk{\xk}} - \muk \etak \left[ \Delta(\xk,\xopt) - 2 \Delta(\xk,\xkh) \right] \\
\Delta(\xkk,\xopt) & \leq \left(1 - \etak \muk \right)\Delta(\xk, \xopt) - \Delta(\xkh, \xk) + \etak^{2} \normsq{\gradk{\xkh} - \gradk{\xk}} + 2 \muk \etak \Delta(\xk,\xkh)
\end{align*}
Now we consider using a constant step-size as well as the Lipschitz line-search.
\subsubsection{Using a constant step-size}
\begin{proof}
\begin{align*}
\intertext{Using smoothness of $\fk(\cdot)$,}
\Delta(\xkk,\xopt) & \leq \left(1 - \etak \muk \right)\Delta(\xk, \xopt) - \Delta(\xkh, \xk) + \etak^{2} \Lk^2 \Delta(\xkh,\xk) + 2 \muk \etak \Delta(\xk,\xkh) \\
\implies \Delta(\xkk,\xopt) & \leq \left( 1 - \etak \muk \right) \Delta(\xk, \xopt) + \left(\etak^{2} \Lk^2 - 1 + 2 \muk \etak \right) \Delta(\xkh, \xk) \\
\intertext{Taking expectation with respect to $i_{k}$,}
\E \left[\Delta(\xkk,\xopt)\right] & \leq \E \left[\left( 1 - \etak \muk \right)\Delta(\xk, \xopt) \right] + \E \left[ \left(\etak^{2} \Lk^2 - 1 + 2 \muk \etak \right) \Delta(\xkh, \xk) \right] \\
\intertext{Note that $\xk$ doesn't depend on $i_k$. Furthermore, neither does $\xopt$ because of the interpolation property.}
\implies \E \left[\Delta(\xkk,\xopt)\right] & \leq \E \left[1 - \etak \muk \right] \Delta(\xk, \xopt) + \E \left[ \left(\etak^{2} \Lk^2 - 1 + 2 \muk \etak \right) \Delta(\xkh, \xk) \right] \\
\intertext{If $\etak \leq \frac{1}{4 \cdot L_{\text{max}}}$, then $\left(\etak^{2} \Lk^2 - 1 + 2 \muk \etak \right) \leq 0$ and}
\implies \E \left[\Delta(\xkk,\xopt)\right] & \leq \E \left[1 - \frac{\muk}{4 L_{\text{max}}} \right] \Delta(\xk, \xopt) \\
\implies \E \left[\Delta(\xkk,\xopt)\right] & \leq \left(1 - \frac{\bar{\mu}}{4 L_{\text{max}}} \right) \Delta(\xk, \xopt) \\
\implies \E \left[\Delta(\x_{k},\xopt)\right] & \leq \left(1 - \frac{\bar{\mu}}{4 L_{\text{max}}} \right)^{T} \Delta(\x_{0}, \xopt)
\end{align*}
\end{proof}
\subsubsection{Using the line-search}
\begin{proof}
\begin{align*}
\intertext{Using Equation~\eqref{eq:lip-ls} to control the difference in gradients,}
\Delta(\xkk,\xopt) & \leq \left(1 - \etak \muk \right)\Delta(\xk, \xopt) - \Delta(\xkh, \xk) + c^2 \Delta(\xkh,\xk) + 2 \muk \etak \Delta(\xk,\xkh) \\
\implies \Delta(\xkk,\xopt) & \leq \left(1 - \etak \muk \right)\Delta(\xk, \xopt) + \left(c^2 + 2 \muk \etak - 1 \right) \Delta(\xkh, \xk) \\
\intertext{Taking expectation with respect to $i_{k}$,}
\E \left[\Delta(\xkk,\xopt) \right] & \leq \E \left[1 - \etak \muk \Delta(\xk, \xopt) \right]  + \E \left[\left(c^2 - 1 + 2 \etak \mu_{i} \right) \Delta(\xkh, \xk) \right] \\
\intertext{Note that $\xk$ doesn't depend on $i_k$. Furthermore, neither does $\xopt$ because of the interpolation property.}
\implies \E \left[\Delta(\xkk,\xopt) \right] & \leq \E \left[1 - \etak \muk \right] \Delta(\xk, \xopt) + \E \left[\left(c^2 - 1 + 2 \etak \muk \right) \Delta(\xkh, \xk) \right] \\
\intertext{Using smoothness, the line-search in Equation~\ref{eq:lip-ls} is satisfied if $\etak \leq \frac{c}{\Lk}$, implying that the step-size returned by the line-search always satisfies $\etak \geq \min \left\{\frac{c}{\Lk}, \eta_{\text{max}} \right\}$.}
\implies \E \left[\Delta(\xkk,\xopt) \right] & \leq \E \left(1 - \muk \min \left\{ \frac{c}{\Lk}, \eta_{\text{max}} \right\} \right) \Delta(\xk, \xopt) + \E \left[\left(c^2 - 1 + 2 \etak \muk \right) \Delta(\xkh, \xk) \right]
\intertext{If we ensure that $\etak \leq \frac{c}{\muk}$, then $c^2 - 1 + 2 \etak \muk \leq 0$. In other words, we need to ensure that $\eta_{\text{max}} \leq \min_{i} \frac{c}{\mu_i}$. Choosing $c = 1/4$, we obtain the following:}
\E \left[\Delta(\xkk,\xopt) \right] & \leq \E \left(1 - \muk \min \left\{ \frac{1}{4 \; \Lk}, \eta_{\text{max}} \right\} \right) \Delta(\xk, \xopt) \\
\intertext{We consider the following cases: $\eta_{\text{max}} < \frac{1}{4 \; L_{\text{max}}}$ and $\eta_{\text{max}} \geq \frac{1}{4 \; L_{\text{max}}}$. When $\eta_{\text{max}} < \frac{1}{4 \; L_{\text{max}}}$, }
\E \left[\Delta(\xkk,\xopt) \right] & \leq \E \left(1 - \muk \; \eta_{\text{max}} \right) \Delta(\xk, \xopt) \\
& = \left(1 - \bar \mu \; \eta_{\text{max}} \right) \Delta(\xk, \xopt) \\
\implies \E \left[\Delta(\xkk,\xopt) \right] & \leq \left(1 - \bar \mu \; \eta_{\text{max}} \right)^{T} \Delta(\x_{0}, \xopt)\\
\intertext{When $\eta_{\text{max}} \geq 1 / (4 \; L_{\text{max}})$, we use $\min \left\{\frac{1}{4 \; \Lk}, \eta_{\text{max}} \right\} \geq \min \left\{\frac{1}{4 \; L_{\text{max}}}, \eta_{\text{max}} \right\}$ to obtain}
\E \left[\Delta(\xkk,\xopt) \right] & \leq \E \left(1 - \muk \min \left\{ \frac{1}{4 \; L_{\text{max}}}, \eta_{\text{max}} \right\} \right) \Delta(\xk, \xopt) \\
& = \E \left(1 - \muk \frac{1}{4 \; L_{\text{max}}} \right) \Delta(\xk, \xopt) \\
& = \left(1 - \frac{\bar \mu }{4 \; L_{\text{max}}} \right) \Delta(\xk, \xopt) \\
\implies \E \left[\Delta(\xkk,\xopt) \right] & \leq \left(1 - \frac{\bar \mu }{4 \; L_{\text{max}}} \right)^{T} \Delta(\x_{0}, \xopt).\\
\intertext{Putting the two cases together, we obtain}
\E \left[\Delta(\xkk,\xopt) \right] & \leq \max \left\{\left(1 - \frac{\bar \mu }{4 \; L_{\text{max}}}\right),\left( 1- \bar \mu \; \eta_{\text{max}} \right)\right\}^{T} \Delta(\x_{0}, \xopt).
\end{align*}
\end{proof}

\subsection{Proof of SEG for convex minimization}
\label{app:seg-min-ls}
\begin{theorem}
Assuming the interpolation property and under $L$-smoothness and convexity of $f$, SEG with Lipschitz line-search with $c = 1/\sqrt{2}$ in Equation~\ref{eq:lip-ls} and iterate averaging achieves the following rate:
\begin{align*}
\E \left[ f(\mxt) - f(\xopt) \right] & \leq \frac{2 \; \max \left\{ \sqrt{2} \; L_{\text{max}}, \frac{1}{\eta_{\text{max}}} \right\} }{T} \normsq{\x_{0} - \xopt} \,.
\end{align*}
Here, $\mxt = \frac{\left[ \sum_{i = 1}^{T} \x_{i} \right]}{T}$ is the averaged iterate after $T$ iterations.
\label{thm:seg-conv-min-ls}
\end{theorem}
\begin{proof}
Starting from Equation \eqref{eq:SEG_intermediate_inequality},
\begin{align*}
2 \etak \left[\langle \gradk{\xkh}, \xkh - \xopt \rangle \right]
& = \Delta(\xk, \xopt) - \Delta(\xkh, \xk) + \etak^{2} \normsq{\gradk{\xkh} - \gradk{\xk}} - \Delta(\xkk,\xopt)\\
\intertext{and using the standard convexity inequality,}
\langle \gradk{\xkh}, \xkh - \xopt \rangle
& \geq f_{i_k}(\xkh) - f_{i_k}(\xopt)  \\
& \geq \tfrac{1}{4}(f_{i_k}(\xkh) - f_{i_k}(\xopt)) \\
& \geq \tfrac{1}{4}(f_{i_k}(\xk) - \etak \|\gradk{\xk}\|^2 - f_{i_k}(\xopt))\\
& = \tfrac{1}{4}(f_{i_k}(\xk) - \frac{1}{\etak} \Delta(\xk,\xkh) - f_{i_k}(\xopt)) \\
\implies 2 \etak \left[\langle \gradk{\xkh}, \xkh - \xopt \rangle \right] & \geq \frac{\etak}{2} \left[ f_{i_k}(\xk) - f_{i_k}(\xopt) \right] - \frac{1}{2} \Delta(\xk,\xkh)
\intertext{ where we used the interpolation hypothesis to say that $\xopt$ is a minimizer of $f_{i_k}$ and thus $f_{i_k}(\xkh) \geq f_{i_k}(\xopt)$. Combining this with \eqref{eq:SEG_intermediate_inequality} and \eqref{eq:lip-ls} leads to,}
\frac{\etak}{2} (f_{i_k}(\xk) - f_{i_k}(\xopt))
& \leq \Delta(\xk,\xopt) - \Delta(\xkk,\xopt)  - \tfrac{1}{2} \Delta(\xkh, \xk) + \etak^{2} \normsq{\gradk{\xkh} - \gradk{\xk}}\\
& \leq \Delta(\xk,\xopt) - \Delta(\xkk,\xopt)  - (\tfrac{1}{2} -c^2 ) \Delta(\xkh, \xk) \\
& \leq \Delta(\xk,\xopt) - \Delta(\xkk,\xopt), \\
\implies f_{i_k}(\xk) - f_{i_k}(\xopt) & \leq \frac{2}{\etak} \left[
\Delta(\xk,\xopt) - \Delta(\xkk,\xopt) \right]
\intertext{where for the last inequality we used Equation \ref{eq:lip-ls} and the fact that $c^2 \leq 1/2$. By definition of the Lipschitz line-search, $\etak \in \left[\min \left\{ c / L_{\text{max}}, \eta_{\text{max}}\right\}, \eta_{\text{max}}\right]$, implying}
\frac{1}{\etak} & \leq \max \left\{ \frac{L_{\text{max}}}{c}, \frac{1}{\eta_{\text{max}}} \right\} \\
\intertext{Setting $c = \frac{1}{\sqrt{2}}$,}
\frac{1}{\etak} & \leq \max \left\{\sqrt{2} L_{\text{max}}, \frac{1}{\eta_{\text{max}}} \right\} \\
f_{i_k}(\xk) - f_{i_k}(\xopt)& \leq  2 \; \max \left\{\sqrt{2}  L_{\text{max}}, \frac{1}{\eta_{\text{max}}} \right\} \left( \Delta(\xk,\xopt) - \Delta(\xkk,\xopt) \right) \\
\intertext{ Taking expectation with respect to $i_{k}$,}
f(\xk) - f(\xopt) & \leq 2 \; \max \left\{\sqrt{2}  L_{\text{max}}, \frac{1}{\eta_{\text{max}}} \right\}  \left( \Delta(\xk,\xopt) - \E \Delta(\xkk,\xopt)  \right)
\intertext{Finally, taking the expectation respect to $\xk$ and summing for $k=1,\ldots,T$, we get,}
& \E \left[ f(\bar \x_k) - f(\xopt) \right] \leq  \frac{2 \; \max \left\{\sqrt{2}  L_{\text{max}}, \frac{1}{\eta_{\text{max}}} \right\} \; \Delta(\x_{0}, \xopt)  }{T}
\end{align*}
\end{proof}

\subsection{SEG for general strongly monotone operators}
\label{app:seg-sc-minmax}
Let $F(\cdot)$ be a Lipschitz  (strongly)-monotone operator. $F$ satisfies the following inequalities for all $u$, $v$, 
\begin{align*}
\norm{F(u) - F(v)} \leq L \norm{u - v} & \tag{Lipschitz continuity} \\
\langle F(u) - F(v), u - v \rangle \geq \mu \normsq{u - v}  & \tag{Strong monotonicity}  
\end{align*}
Here, $\mu$ is the strong-monotonicity constant and $L$ is the Lipschitz constant. Note that $\mu = 0$ for monotone operators. We seek the solution $\xopt$ to the following optimization problem: $\sup_{\x} \langle F(\xopt), \xopt - \x \rangle \leq 0$. 

Note that for strongly-convex minimization where $\xopt = \argmin f(\x)$, $F$ is equal to the gradient operator and $\mu$ and $L$ are the strong-convexity and smoothness constants in the previous sections.

SEG~\cite{juditsky2011solving} is a common method for optimizing stochastic variational inequalities and results in an $O(1/\sqrt{T})$ rate for monotone operators and an $O(1/T)$ rate for strongly-monotone operators~\cite{gidel2018variational}. For strongly-monotone operators, the convergence can be improved to obtain a linear rate by using variance-reduction methods~\cite{palaniappan2016stochastic, chavdarova2019reducing} exploiting the finite-sum structure in $F$. In this setting, $F(\x) = \frac{1}{n} \sum_{i = 1}^{n} F_{i}(\x)$. To the best of our knowledge, the interpolation condition has not been studied in the context of general strongly monontone operators. In this case, the interpolation condition implies that $F_i(\xopt) = 0$ for all operators $F_i$ in the finite sum.
\begin{theorem}[Strongly-monotone]
Assuming (a) interpolation, (b) $L$-smoothness and (c) $\mu$-strong monotonocity of $F$, SEG using Lipschitz line-search with $c = 1/4$ in Equation~\ref{eq:lip-ls} and setting $\eta_{\text{max}} \leq \min_i \frac{1}{4 \mu_i}$ has the rate:
\begin{align*}
\E \left[ \normsq{\x_{k} - \xopt} \right] & \leq \left( \max \left\{ \left(1 - \frac{\bar{\mu}}{4 \; L_{\text{max}}} \right),\left( 1 - \eta_{\text{max}} \; \bar{\mu} \right) \right\} \right)^{T} \normsq{\x_{0} - \xopt} \,.
\end{align*}
\label{thm:seg-sc-minmax}
\end{theorem}
\begin{proof}
\begin{align*}
    \intertext{For each $\opk{\cdot}$, we use the strong-monotonicity condition with constant $\muk$,}
    \langle \opk{u} - \opk{v}, u - v \rangle & \geq \muk \normsq{u - v} \\
    \intertext{Set $u = \x$, $v = \xopt$,}
    \implies \langle \opk{\x} - \opk{\xopt}, \x - \xopt \rangle & \geq \muk \normsq{\x - \xopt} \\
    \intertext{By the interpolation condition,}
    \opk{\xopt} & = 0 \\
    \implies \langle \opk{\x}, \x - \xopt \rangle & \geq \muk
    \normsq{\x - \xopt} \\
    \intertext{This is equivalent to an RSI-like condition, but with the gradient operator $\gradk{\cdot}$ replaced with a general operator $\opk(\cdot)$.}
\end{align*}
From here on, the theorem follows the same proof as that for Theorem~\ref{thm:seg-min-ls} above with the $\opk$ instead of $\gradk$ and the strong-convexity constant being replaced with the constant for strong-monotonicity.
\end{proof}
Like in the RSI case, the above result can also be obtained using a constant step-size $\eta \leq \frac{1}{4 \; L_{\text{max}}}$.

\subsection{SEG for bilinear saddle point problems}
\label{sec:seg-bilin-ls}
Let us consider the bilinear saddle-point problem of the form $\min_{x} \max_{y} x^\top A y - x^\top b - y^\top c$, where $A$ is the ``coupling'' matrix and where both $b$ and $c$ are vectors~\cite{gidel2018variational, chavdarova2019reducing}. In this case, the (monotone) operator $F(x,y) = \left[ Ax- b, \, -A^\top y +c \right]$ and we assume the finite sum formulation as:
\begin{equation}
    x^\top A y - x^\top b - y^\top c = \frac{1}{n} \sum_{i=0}^n x^\top A_i y - x^\top b_i - y^\top c_i
\end{equation}
We show that the interpolation condition enables SEG with Lipschitz line-search achieve a linear rate of convergence. In every iteration, the SEG algorithm samples rows $A_i$ (resp. columns $A_j$) of the matrix $A$ and the respective coefficient $b_i$ (resp. $c_j$). If $x_k$ and $y_k$ correspond to the iterates for the minimization and maximization problem respectively, then the update rules for SEG can be written as:
\begin{equation}
    \left\{\begin{aligned}
    x_{k+1} = x_k - \etak (A_{i_k} y_{k+1/2} - b_{i_k})\\
    y_{k+1} = y_k + \etak (A_{i_k}^\top x_{k+1/2} - c_{i_k})
    \end{aligned}
    \right.
    \quad 
    \text{and}
    \quad
        \left\{\begin{aligned}
    x_{k+1/2} = x_k - \etak (A_{i_k} y_k  - b_{i_k})\\
    y_{k+1/2} = y_k + \etak (A_{i_k}^\top x_k  - c_{i_k})
    \end{aligned}
    \right.
\end{equation}
which can be more compactly written as,
\begin{align}
\label{eq:seg_bilin_update}
    x_{k+1} = x_k - \etak (A_{i_k} (y_k + \etak(A_{i_k}^\top x_k - c_{i_k}) - b_{i_k})\\
    y_{k+1} = y_k + \etak (A_{i_k}^\top (x_k - \etak(A_{i_k} y_k - b_{i_k}) - c_{i_k}) \, \nonumber.
\end{align}
We now prove that SEG attains the following linear rate of convergence. 
\begin{theorem}[Bilinear]
Assuming the (a) interpolation property and for the (b) bilinear saddle point problem, SEG with Lipschitz line-search with $c = 1/\sqrt{2}$ in Equation~\ref{eq:lip-ls} achieves the following rate:
\begin{align*}
\E \left[ \normsq{\x_{k} - \xopt} \right] & \leq \left( \max \left\{ \left(1- \frac{\sigma_{\min}(\E[A_{i_k} A^\top_{i_k}]}{4\max_i \sigma_{\text{max}}(A_i A_i^\top)}\right), \left(1- \frac{\eta_{\text{max}}}{2} \; \sigma_{\min}(\E[A_{i_k} A^\top_{i_k}]\right) \right\}\right)^{T} (\|x_{k}\|^2 + \|y_{k}\|^2 )
\end{align*}
\label{thm:seg-bilin-ls}
\end{theorem}
\begin{proof}
If $(x^*, y^*)$ is the solution to the above saddle point problem, then interpolation hypothesis implies that 
\begin{equation*}
    A_{i_k} y^* = b_{i_k} \quad \text{and} \quad A_i^\top x^* = c_i
\end{equation*}
We note that the problem can be reduced to the case $b = c =0$ by using the change of variable $\tilde x_k := x_k - x^*$ and $\tilde y_k := y_k - y^*$. 
\begin{align*}
    \tilde x_{k+1} = x_{k+1} - x^* 
    &= x_k -x^* - \etak (A_{i_k} (y_k - y^* + \etak A_{i_k}^\top (x_k - x^*)) 
    = \tilde x_k - \etak A_{i_k} (\tilde y_k + \etak A_{i_k}^\top \tilde x_k) 
    \\
    \tilde y_{k+1} = y_{k+1} - y^* 
    &= y_k - y^* + \etak (A_{i_k}^\top (x_k - x^* - \etak A_{i_k} (y_k - y^*))
    = \tilde y_k + \etak A_{i_k}^\top (\tilde x_k - \etak A_{i_k} \tilde y_k)
\end{align*}
Thus, $(\tilde x_{k+1}, \tilde y_{k+1})$ correspond to the update rule~Eq.\eqref{eq:seg_bilin_update} with $b = c = 0$. Note that the interpolation hypothesis is key for this problem reduction.

In the following, without loss of generality, we will assume that $b = c =0$. 

Using the update rule, we get,
\begin{align*}
    \|x_{k+1}\|^2 + \|y_{k+1}\|^2 à
    &=\|x_{k}\|^2 + \|y_{k}\|^2 -  \etak^2 ( x_k^\top A_{i_k} A_{i_k}^\top x_k + y_k^\top A_{i_k}^\top A_{i_k} y_k) + \etak^4 ( x_k^\top (A_{i_k} A_{i_k}^\top)^2 x_k + y_k^\top (A_{i_k}^\top A_{i_k})^2 y_k)
\end{align*}
The line-search hypothesis can be simplified as,
\begin{equation}
    \etak^2 ( x_k^\top (A_{i_k} A_{i_k}^\top)^2 x_k + y_k^\top (A_{i_k}^\top A_{i_k})^2 y_k)
    \leq c^2 ( x_k^\top A_{i_k} A_{i_k}^\top x_k + y_k^\top A_{i_k}^\top A_{i_k} y_k)
\end{equation}
leading to,
\begin{align*}
    \|x_{k+1}\|^2 + \|y_{k+1}\|^2 à
    & \leq \|x_{k}\|^2 + \|y_{k}\|^2 -  \etak^2(1-c^2) ( x_k^\top A_{i_k} A_{i_k}^\top x_k + y_k^\top A_{i_k}^\top A_{i_k} y_k)
\end{align*}
Noting that $L_{\text{max}} = \left[\max_i \sigma_{\text{max}}(A_i A_i^\top) \right]^{1/2}$, we obtain $\etak \geq  \min \left\{ \left[2 \max_i \sigma_{\text{max}}(A_i A_i^\top)\right]^{-1/2}, \eta_{\text{max}} \right\}$ from the Lipschitz line-search. Taking the expectation with respect to $i_k$ gives,
\begin{align*}
    \E \big[ \|x_{k+1}\|^2 + \|y_{k+1}\|^2 \big]
    & \leq (1- \etak^2 \sigma_{\min}(\E[A_{i_k} A^\top_{i_k}])(1-c^2)) (\|x_{k}\|^2 + \|y_{k}\|^2 )\\
    &\leq \max \left\{ \left(1- \frac{\sigma_{\min}(\E[A_{i_k} A^\top_{i_k}])}{4 \max_i \sigma_{\text{max}}(A_i A_i^\top)}\right), \left(1 - \frac{\eta_{\text{max}}}{2} \; \sigma_{\min}(\E[A_{i_k} A^\top_{i_k}])\right) \right\} (\|x_{k}\|^2 + \|y_{k}\|^2 ).
\end{align*}
Applying this inequality recursively and taking expectations yields the final result.
\end{proof}
Observe that the rate depends on the minimum and maximum singular values of the matrix formed using the mini-batch of examples selected in the SEG iterations. Note that these are the first results for bilinear min-max problems in the stochastic, interpolation setting.

%% file: App-Proofs-SGD-nc-additional.tex
\section{Additional Non-Convex Proofs}
\label{app:add_nc_proofs}

We can prove additional convergence results for non-convex functions satisfying the SGC by (1) enforcing independence between $\etak$ and the stochastic gradient $\gradk{\xk}$ or (2) enforcing that the sequence of step-sizes $(\etak)$ is non-increasing. \\

\subsection{Independent Step-sizes}

To enforce independence of the step-size and stochastic gradient, we perform a backtracking line-search at the current iterate $\xk$ using a mini-batch of examples that is \emph{independent} of the mini-batch on which $\gradk{\xk}$ is evaluated. 
This equates to evaluating the Armijo condition in Equation~\ref{eq:c-ls} using stochastic estimates $f_{jk}(\xk)$ and $\nabla f_{jk}(\xk)$ that are independent of $\gradk{\xk}$, which is used in the gradient update.
In practice, the mini-batch at previous iteration can be re-used for the line-search.
A similar technique has recently been used to prove convergence of SGD with the AdaGrad step-size~\cite{li2019convergence}.
Alternatively, the current mini-batch can be divided, with one set of examples used to evaluate the stochastic Armijo condtion and the other to compute the gradient-step.
In this setting, we prove the following convergence rate for SGD with the Armijo line-search.

\begin{theorem}
Assuming (a) the SGC with constant $\rho$, (b) $L_i$-smoothness of $f_i$'s, and (c) independence of $\etak$ and the stochastic gradient $\gradk{\xk}$ at every iteration $k$, SGD with the Armijo line-search in Equation~\ref{eq:c-ls} with $c = 1/2$ and setting $\eta_{\text{max}} = \nicefrac{1}{\rho L}$ achieves the rate:
\begin{align*}
\min_{k = 0, \ldots, T-1} \E \normsq{\grad{\xk}} &\leq \frac{2}{T} \left(\max\left\{L_{\text{max}}, \frac{1}{\eta_{\text{max}}}\right\}\right) \left( f(\x_0) - f(\xopt)\right).
\end{align*}
\label{thm:sgd-decor-nc}
\end{theorem}

\begin{proof}
Starting from L-smoothness of $f$,
\begin{align*}
    f(\xkk) &\leq f(\xk) + \langle \grad{\xk}, \xkk - \xk \rangle + \frac{L}{2}\normsq{\xkk - \xk}\\
    &= f(\xk) - \etak \langle \grad{\xk}, \gradk{\xk} \rangle + \frac{L \etak^2}{2} \normsq{\gradk{\xk}}\\
    \intertext{Taking expectations with respect to $ik$ and noting that $\etak$ is independent of $\gradk{\xk}$,}
    \implies \E_{ik}\left[f(\xkk)\right] &\leq f(\xk) - \etak \E_{ik}\left[\langle \grad{\xk}, \gradk{\xk} \rangle \right] + \frac{L\etak^2}{2}\E_{ik}\left[ \normsq{\gradk{\xk}}\right]\\
    &= f(\xk) - \etak \normsq{\grad{\xk}} + \frac{L\etak^2}{2} \E_{ik}\left[ \normsq{\gradk{\xk}} \right]\\
    &\leq f(\xk) - \etak \normsq{\grad{\xk}} + \frac{L\etak^2}{2} \rho \normsq{\grad{\xk}} \tag{by SGC}\\
    &= f(\xk) + \etak \left(\frac{\rho L \etak}{2} - 1 \right) \normsq{\grad{\xk}}\\
    &\leq f(\xk) + \etak \left(\frac{\rho L \eta_{\text{max}}}{2} - 1 \right) \normsq{\grad{\xk}}\\
    \intertext{Assume $\eta_{\text{max}} \leq \frac{2}{\rho L}$. Then $\frac{\rho L \eta_{\text{max}}}{2} - 1\leq 0$ and we may lower bound $\etak$ using Equation~\ref{eq:eta-bounds},}
    \E_{ik}\left[f(\xkk)\right] &\leq f(\xk) + \min
    \left\{\frac{2(1-c)}{L_{\text{max}}}, \eta_{\text{max}}\right\} \left(\frac{\rho L \eta_{\text{max}}}{2} - 1 \right) \normsq{\grad{\xk}}.\\
    \intertext{Choosing $c = \frac{1}{2}$,}
    \E_{ik}\left[f(\xkk)\right] &= f(\xk) + \min\left\{\frac{1}{L_{\text{max}}}, \eta_{\text{max}}\right\} \left(\frac{\rho L \eta_{\text{max}}}{2} - 1 \right) \normsq{\grad{\xk}}\\
    \intertext{Taking expectation w.r.t $\etak$ and $\xk$,}
    \E\left[f(\xkk)\right] &\leq f(\xk) + \min\left\{\frac{1}{L_{\text{max}}}, \eta_{\text{max}}\right\} \left(\frac{\rho L \eta_{\text{max}}}{2} - 1 \right) \normsq{\grad{\xk}} \addtocounter{equation}{1}\tag{\theequation} \label{eq:decorr-intermediate} \\
    \implies \normsq{\grad{\xk}} &\leq \max\left\{L_{\text{max}}, \frac{1}{\eta_{\text{max}}}\right\} \left(\frac{2}{2 - \rho L \eta_{\text{max}}}\right)\left(\E\left[f(\xk) - f(\xkk)\right] \right)
    \intertext{Summing over iterations,}
    \implies \sum_{k = 0}^{T-1} \E \left[ \normsq{\grad{\xk}}\right] &\leq \max\left\{L_{\text{max}}, \frac{1}{\eta_{\text{max}}}\right\} \left(\frac{2}{2 - \rho L \eta_{\text{max}}}\right) \E \left[ \sum_{k=0}^{T-1} \left(f(\xk) - f(\xkk) \right) \right]\\
    &\leq \max\left\{L_{\text{max}}, \frac{1}{\eta_{\text{max}}}\right\} \left(\frac{2}{2 - \rho L \eta_{\text{max}}}\right) \E \left[ f(\x_0) - f(\x_{T}) \right]\\
    &\leq \max\left\{L_{\text{max}}, \frac{1}{\eta_{\text{max}}}\right\} \left(\frac{2}{2 - \rho L \eta_{\text{max}}}\right) \left( f(\x_0) - f(\xopt)\right)\\
    \implies \min_{k = 0, \ldots, T-1} \E \normsq{\grad{\xk}} &\leq \frac{1}{T} \left(\frac{2\max\left\{L_{\text{max}}, \frac{1}{\eta_{\text{max}}}\right\}}{2 - \rho L \eta_{\text{max}}}\right) \left( f(\x_0) - f(\xopt)\right)
    \intertext{If $\eta_{\text{max}} = \frac{1}{\rho L}$,}
    \min_{k = 0, \ldots, T-1} \E \normsq{\grad{\xk}} &\leq \frac{2}{T} \left(\max\left\{L_{\text{max}}, \frac{1}{\eta_{\text{max}}}\right\}\right) \left( f(\x_0) - f(\xopt)\right)
\end{align*}
\end{proof}
\vspace{-1em}
We can extend this result to obtain a linear convergence rate when $f$ satisfies the Polyak-Łojasiewicz inequality with constant $\mu$. 
This is formalized in the following theorem.

\begin{theorem}
Assuming (a) the SGC with constant $\rho$, (b) $L_i$-smoothness of $f_i$'s, (c) $f$ is $\mu$-PL, and (d) independence of $\etak$ and the stochastic gradient $\gradk{\xk}$ at every iteration $k$, SGD with the Armijo line-search in Equation~\ref{eq:c-ls} with $c = 1/2$ and setting $\eta_{\text{max}} = \nicefrac{1}{\rho L}$ achieves the rate:
\begin{align*}
\E\left[f(\x_{T}) - f(\xopt)\right] &\leq  \left(1 - \mu \min\left\{\frac{1}{L_{\text{max}}}, \eta_{\text{max}}\right\} \right)^{T} \left(f(\x_{0}) - f(\xopt) \right)
\end{align*}
\label{thm:sgd-decor-nc-pl}
\end{theorem}
\begin{proof}
Starting from Equation~\ref{eq:decorr-intermediate} and applying the PL-inequality:
\begin{align*}
    \E\left[f(\xkk)\right] &\leq f(\xk) - \min\left\{\frac{1}{L_{\text{max}}}, \eta_{\text{max}}\right\} \left(1 - \frac{\rho L \eta_{\text{max}}}{2} \right) \normsq{\grad{\xk}}\\
    &\leq f(\xk) - \min\left\{\frac{1}{L_{\text{max}}}, \eta_{\text{max}}\right\} \left(1 - \frac{\rho L \eta_{\text{max}}}{2} \right) 2 \mu \left(f(\xk) - f(\xopt)\right) \\
    \implies \E\left[f(\xkk)\right] - f(\xopt) &\leq  \left(1 - 2\mu \min\left\{\frac{1}{L_{\text{max}}}, \eta_{\text{max}}\right\} \left(1 - \frac{\rho L \eta_{\text{max}}}{2} \right) \right) \left(f(\xk) - f(\xopt) \right)
    \intertext{If $\eta_{\text{max}} = \frac{1}{\rho L}$,}
    \E\left[f(\xkk)\right] - f(\xopt) &\leq  \left(1 - \mu \min\left\{\frac{1}{L_{\text{max}}}, \eta_{\text{max}}\right\} \right) \left(f(\xk) - f(\xopt) \right)\\
    \intertext{Taking expectations and recursing gives the final result,}
    \implies \E\left[f(\x_{T}) - f(\xopt)\right] &\leq  \left(1 - \mu \min\left\{\frac{1}{L_{\text{max}}}, \eta_{\text{max}}\right\} \right)^{T} \left(f(\x_{0}) - f(\xopt) \right).
\end{align*}
\end{proof}
\vspace{-1em}
\subsection{Non-increasing Step-sizes}

An alternative approach is to assume that the sequence of step-sizes $(\etak)$ is non-increasing. That is, $\etak \leq \eta_{k-1}$ for all $k$.
This holds, for example, when using reset option 0 in Algorithm~\ref{alg:reset_options}. 
In this case, we prove a $O(1/T)$ convergence result for SGD with the stochastic Armijo line-search.
Our result assumes that the iterates $(\xk)$ are contained in a bounded set $\mathcal{X}$ centered on $\xopt$ and depends on the diameter of this set,
\[ D = \max_{\x \in \mathcal{X}} \norm{\x - \xopt}. \]
This is formalized as follows.

\begin{theorem}
Assuming (a) the SGC with constant $\rho$, (b) $L_i$-smoothness of $f_i$'s, (c) $\etak$ are non-increasing, and (d) $\norm{\xk - \xopt} \leq D$ for all $k$, SGD with the Armijo line-search in Equation~\ref{eq:c-ls} with $c = 1/2$ and setting $\eta_{\text{max}} = \nicefrac{1}{\rho L}$ achieves the rate:
\begin{align*}
\min_{k = 0, \ldots, T-1} \E \normsq{\grad{\xk}} &\leq \frac{L\,D^2}{T} \max\left\{L_{\text{max}}, \rho L \right\},
\end{align*}
where $C = \max\left\{ f(\xk) : k = 0, \dots, T-1 \right\}$
\label{thm:sgd-monotone-nc}
\end{theorem}

\begin{proof}
Starting from L-smoothness of $f$,
\begin{align*}
    f(\xkk) & \leq f(\xk) - \langle \grad{\xk}, \etak \gradk{\xk} \rangle + \frac{L \etak^2}{2} \normsq{\gradk{\xk}} \\
\frac{f(\xkk) - f(\xk)}{\etak} & \leq - \langle \grad{\xk}, \gradk{\xk} \rangle + \frac{L \etak}{2} \normsq{\gradk{\xk}} \\
\intertext{Taking expectation,}
\E \left[\frac{f(\xkk) - f(\xk)}{\etak} \right] & \leq -\normsq{\grad{\xk}} + \E \left[\frac{L \etak}{2} \normsq{\gradk{\xk}}\right] \\    
& \leq -\normsq{\grad{\xk}} + \frac{L \eta_{\text{max}}}{2} \E \left[ \normsq{\gradk{\xk}} \right] \\    
\implies \E \left[\frac{f(\xkk) - f(\xk)}{\etak} \right] & \leq - \normsq{\grad{\xk}} + \frac{L \eta_{\text{max}} \rho}{2}  \normsq{\grad{\xk}} & \tag{By the SGC} \\    
\implies \left(1 - \frac{L \eta_{\text{max}} \rho}{2} \right) \normsq{\grad{\xk}} & \leq \E \left[\frac{f(\xk) - f(\xkk)}{\etak} \right] \\
\intertext{If $\eta_{\text{max}} \leq \frac{2}{L \rho}$,}
\normsq{\grad{\xk}} & \leq \frac{1}{1 - \frac{L \eta_{\text{max}} \rho}{2}} \E\left[\frac{f(\xk) - f(\xkk)}{\etak}\right]. 
\intertext{Taking expectations and summing from $k=0$ to $T-1$,}
\sum_{k=0}^{T-1} \E \left[\normsq{\grad{\xk}}\right] & \leq \frac{1}{1 - \frac{L \eta_{\text{max}} \rho}{2}} \E\left[ \sum_{k=0}^{T-1} \frac{f(\xk) - f(\xkk)}{\etak}\right] \\
&= \frac{1}{1 - \frac{L \eta_{\text{max}} \rho}{2}} \E\left[\frac{f(\x_0)}{\eta_0} - \frac{f(\x_T)}{\eta_{T-1}} + \sum_{k=0}^{T-2} \left(\frac{1}{\eta_{k+1}} - \frac{1}{\eta_{k}}\right)f(\xkk) \right].
\intertext{Denoting $C = \max\left\{ f(\xk) : k = 0 \dots T-1 \right\}$ and recalling $\eta_{k+1} \leq \etak$ for all $k$ gives}
\sum_{k=0}^{T-1} \E \left[\normsq{\grad{\xk}}\right]  &\leq \frac{1}{1 - \frac{L \eta_{\text{max}} \rho}{2}} \E\left[\frac{f(\x_0)}{\eta_0} - \frac{f(\x_T)}{\eta_{T-1}} + C \sum_{k=0}^{T-2} \left(\frac{1}{\eta_{k+1}} - \frac{1}{\eta_{k}}\right) \right]\\
&= \frac{1}{1 - \frac{L \eta_{\text{max}} \rho}{2}} \E\left[ \frac{f(\x_0)}{\eta_0} - \frac{f(\x_T)}{\eta_{T-1}} + C \left( \frac{1}{\eta_{T-1}} - \frac{1}{\eta_{0}} \right) \right]\\
&= \frac{1}{1 - \frac{L \eta_{\text{max}} \rho}{2}} \E\left[ \frac{f(\x_0) - C}{\eta_0} + \frac{C - f(\xopt)}{\eta_{T-1}} \right] \\ 
&\leq \frac{1}{1 - \frac{L \eta_{\text{max}} \rho}{2}} \E\left[\frac{C - f(\xopt)}{\eta_{T-1}} \right]. \tag{since $f(\x_0) - C \leq 0$}\\
\intertext{Noting that $C - f(\xopt) \geq 0$, we may bound $1/\eta_{T-1}$ using Equation~\ref{eq:eta-bounds} as follows:}
\frac{1}{\eta_{T-1}} & \leq \max\left\{\frac{\Lk}{2 \; (1-c)}, \frac{1}{\eta_{\text{max}}} \right\} \leq \max\left\{\frac{L_{\text{max}}}{2 \; (1-c)}, \frac{1}{\eta_{\text{max}}} \right\}. \\
\intertext{This implies}
\sum_{k=0}^{T-1} \E \left[\normsq{\grad{\xk}}\right] &\leq \frac{1}{1 - \frac{L \eta_{\text{max}} \rho}{2}} \max\left\{\frac{L_{\text{max}}}{2(1-c)}, \frac{1}{\eta_{\text{max}}} \right\} \E\left[C - f(\xopt) \right]\\
\implies \min_{k = 0, \ldots, T-1} \E \normsq{\grad{\xk}} &\leq \frac{1}{T} \left(\frac{1}{1 - \frac{L \eta_{\text{max}} \rho}{2}}\right) \max\left\{\frac{L_{\text{max}}}{2(1-c)}, \frac{1}{\eta_{\text{max}}} \right\} \E\left[C - f(\xopt) \right].\\ 
\intertext{If $\eta_{\text{max}} = \frac{1}{\rho L}$ and $c = \frac{1}{2}$, }
\min_{k = 0, \ldots, T-1} \E \normsq{\grad{\xk}} &\leq \frac{2}{T} \max\left\{L_{\text{max}}, \rho L \right\} \E\left[C - f(\xopt) \right].
\intertext{By $L$-smoothness of $f$,}
f(\xk) - f(\xopt) &\leq \frac{L}{2} \normsq{\xk - \xopt} \leq \frac{L}{2} D^2 \tag{since $\normsq{\xk - \xopt} \leq D^2$} \\
\implies C &\leq \frac{L\,D^2}{2},
\intertext{and we conclude}
\min_{k = 0, \ldots, T-1} \E \normsq{\grad{\xk}} &\leq \frac{L\,D^2}{T} \max\left\{L_{\text{max}}, \rho L \right\}.
\end{align*}
\end{proof}

%% file: App-Experimental-Details.tex
\section{Additional Experimental Details}
\label{app:exp-details}
In this section we give details for all experiments in the main paper and the additional results given in Appendix~\ref{app:additional-results}. In all experiments, we used the default learning rates provided in the implementation for the methods we compare against. For the proposed line-search methods and for \emph{all} experiments in this paper, we set the initial step-size $\eta_{\text{max}} = 1$ and use back-tracking line-search where we reduce the step-size by a factor of $0.9$ if the line-search is not satisfied. We used $c = 0.1$ for all our experiments with both Armijo and Goldstein line-search procedures, $c=0.9$ for SEG with Lipschitz line-search, and $c = 0.5$ when using Nesterov acceleration~\footnote{Note that these choices are inspired by the theory}. For Polyak acceleration, we use $c=0.1$ in our experiments with deep neural networks and $c=0.5$ otherwise. For our non-convex experiments, we always constrain the step-size to be less than $10$ to prevent it from becoming unbounded. Note that we conduct a robustness study to quantify the influence of the $c$ and $\eta_{\text{max}}$ parameter in Section~\ref{app:experiments-robustness-computation}. For the heuristic in~\cite{schmidt2017minimizing, schmidt2015non}, we set the step-size increase factor to $\gamma = 1.5$ for convex minimization and use  $\gamma = 2$ for non-convex minimization. Similarly, when using Polyak momentum we set the momentum factor to the highest value that does not lead to divergence. It is set to $\beta = 0.8$ in the convex case and $\beta = 0.6$ in the non-convex case~\footnote{We hope to use method such as~\cite{zhang2017yellowfin} to automatically set the momentum parameter in the future.}. 

\subsection{Synthetic Matrix Factorization Experiment}
In the following we give additional details for synthetic matrix factorization experiment in Section \ref{sec:experiments-synthetic}. As stated in the main text, we set $A \in \mathbb{R}^{10 \times 6}$ with condition number $\kappa(A) = 10^{10}$ and generated a fixed dataset of $1000$ samples using the code released by Ben Recht~\footnote{This code is available at \url{https://github.com/benjamin-recht/shallow-linear-net}}. We withheld $200$ of these examples as a test set. All optimizers used mini-batches of $100$ examples and were run for $50$ epochs. We averaged over $20$ runs with different random seeds to control for variance in the training loss, which approached machine precision for several optimizers. 

\subsection{Binary Classification using Kernel Methods}
\label{app:kernel_exp_details}
We give additional details for the experiments on binary classification with RBF kernels in Section \ref{sec:experiments-kernels}. For all datasets, we used only the training sets available in the \texttt{LIBSVM} \cite{libsvm} library and used an 80:20 split of it. The 80 percent split of the data was used as a training set and 20 percent split as the test set. The bandwidth parameters for the RBF kernel were selected by grid search using 10-fold cross-validation on the training splits. The grid of kernel bandwidth parameters that were considered is \texttt{[0.05, 0.1, 0.25, 0.5, 1, 2.5, 5, 10, 15, 20]}. For the cross-validation, we used batch L-BFGS to minimize both objectives on the rcv1 and mushrooms datasets, while we used the Coin-Betting algorithm on the larger w8a and ijcnn datasets with mini-batches of 100 examples. In both cases, we ran the optimizers for 100 epochs on each fold. The bandwidth parameters that maximized cross-validated accuracy were selected for our final experiments. The final kernel parameters are given in Table \ref{table:kernel_datasets_details}, along with additional details for each dataset.

\begin{table}[]
    \centering
    \begin{tabular}{c|c c c c c c}
        Dataset & Dimension ($d)$ & Training Set Size  & Test Set Size & Kernel Bandwidth & SVRG Step-Size\\ \hline
        mushrooms & $112$ & $6499$ & $1625$ & $0.5$ & $500$  \\ 
        ijcnn & $22$ & $39992$ & $9998$ & $0.05$ & $500$ \\ 
        rcv1 & $47236$ & $16194$ & $4048$ & $0.25$ & $500$ \\ 
        w8a & $300$ & $39799$ & $9950$ & $20.0$ & $0.0025$ \\ 
    \end{tabular}
    \caption{Additional details for binary classification datasets used in convex minimization experiments. Kernel bandwidths were selected by $10$-fold cross validation on the training set. SVRG step-sizes were selected by 3-fold CV on the training set. See text for more details.}
    \label{table:kernel_datasets_details}
\end{table}
We used the default hyper-parameters for all baseline optimizers used in our other experiments. For PLS, we used the exponential exploration strategy and its default hyper-parameters. Fixed step-size SVRG requires that the step-size parameter to be well-tuned in order to obtain a fair comparison with adaptive methods. To do so, we selected step-sizes by grid search. For each step-size, a 3-fold cross-validation experiment was run on each dataset's training set. On each fold, SVRG was run with mini-batches of size $100$ for $50$ epochs. Final step-sizes were selected by maximizing convergence rate on the cross-validated loss. The grid of possible step-sizes was expanded whenever the best step-size found was the largest or smallest step-size in the considered grid. We found that the mushrooms, ijcnn, and rcv1 datasets admitted very large step-sizes; in this case, we terminated our grid-search when increasing the step-size further gave only marginal improvement. The final step-sizes selected by this procedure are given in Table \ref{table:kernel_datasets_details}.

Each optimizer was run with five different random seeds in the final experiment. All optimizers used mini-batches of $100$ examples and were run for $35$ epochs. Experiment figures display shaded error bars of one standard-deviation from the mean. Note that we did not use a bias parameter in these experiments. 

\subsection{Multi-class Classification using Deep Networks}
\label{app:deep_exp_details}
For mutliclass-classification with deep networks, we considered the MNIST and CIFAR10 datasets, each with $10$ classes. For MNIST, we used the standard training set consisting of $60$k examples and a test set of $10$k examples; whereas for CIFAR10, this split was $50$k training examples and $10$k examples in the test set. As in the kernel experiments, we evaluated the optimizers using the softmax. All optimizers were used with their default learning rates and without any weight decay. We used the experimental setup proposed in~\cite{luo2019adaptive} and used a batch-size of $128$ for all methods and datasets. As before, each optimizer was run with five different random seeds in the final experiment. The optimizers were run until the performance of most methods saturated; $100$ epochs for MNIST and $200$ epochs for the models on the CIFAR10 dataset. We compare against a tuned SGD method, that uses a constant step-size selected according to a search on the ${[1e-1, 1e-5]}$ grid and picking the variant that led to the best convergence in the training loss. This procedure resulted in choosing a step-size of $0.01$ for the MLP on MNIST and $0.1$ for both models on CIFAR10. 

%% file: App-Additional-Results.tex
\section{Additional Results}
\label{app:additional-results}
\begin{figure}[hbt!]
    \centering
    \includegraphics[width =\textwidth]{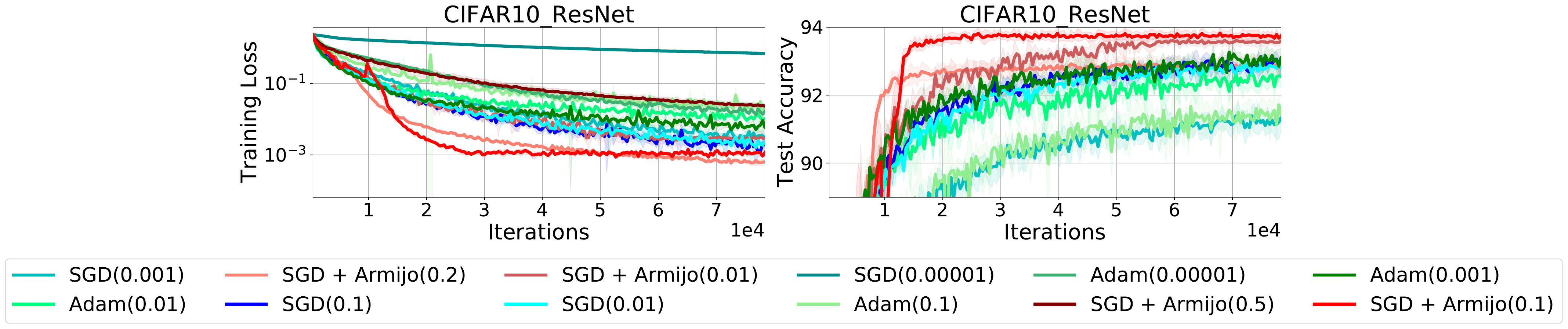}
\caption{Testing the robustness of Adam, SGD and SGD with Armijo line-search for training ResNet on CIFAR10. SGD is highly sensitive to it's fixed step-size; selecting too small a step-size results in very slow convergence. In contrast, SGD + Armijo has similar performance with $c = 0.1$ and $c = 0.01$ and all $c$ values obtain reasonable performance. We note that Adam is similarly robust to its initial learning-rate parameter. }
\label{fig:robustness_ablation}
\end{figure}
\begin{figure}[hbt!]
    \centering
    \includegraphics[width = 0.6\textwidth]{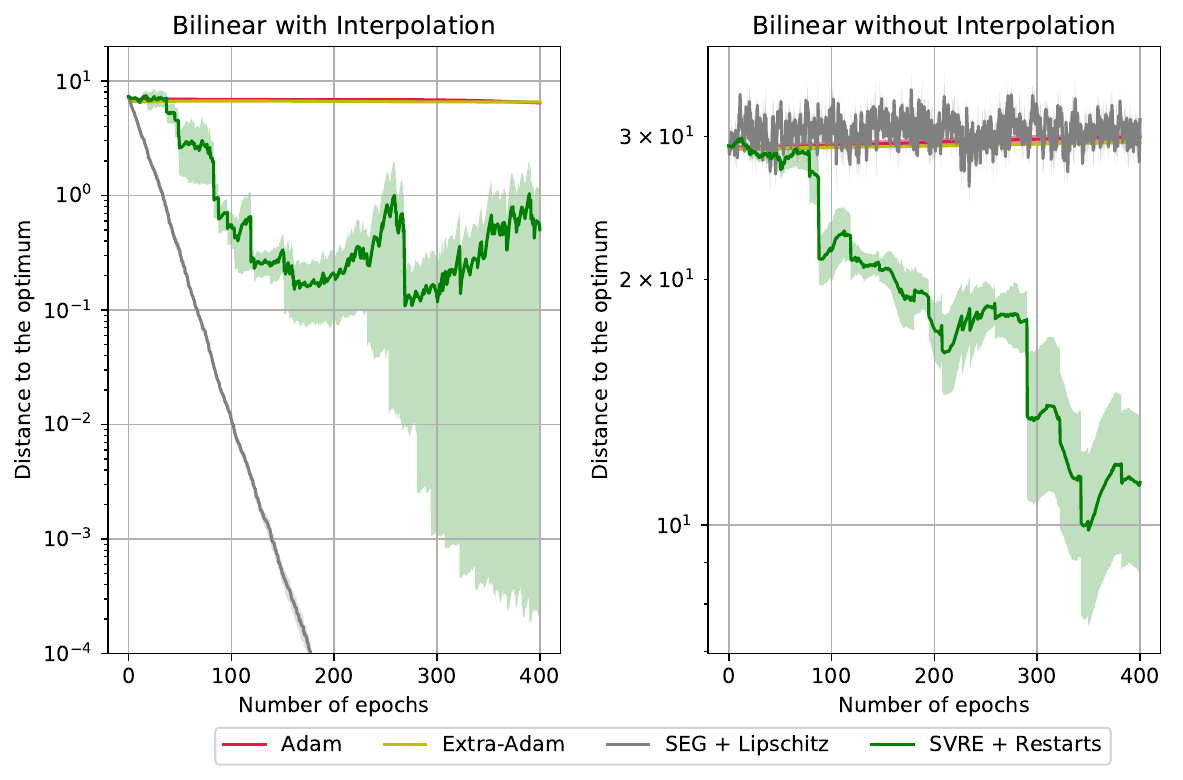}
\caption{Min-max optimization on synthetic bilinear example (left) with interpolation (right) without interpolation. SEG with Lipschitz line-search converges linearly when interpolation is satisfied -- in agreement with in Theorem \ref{thm:seg-bilin-ls} -- although it fails to converge when interpolation is violated. }
\label{fig:bilinear-game}
\end{figure}

\subsubsection{Evaluating robustness and computation}
\label{app:experiments-robustness-computation}
In this experiment, we compare the robustness and computational complexity between the three best performing methods across datasets: Adam, constant step-size and SGD with Armijo line-search. For both Adam and constant step-size SGD, we vary the step-size in the $[10^{-1}, 10^{-5}]$ range; whereas for the SGD with line-search, we vary the parameter $c$ in the $[0.1, 0.5]$ range and vary $\eta_{max} \in [1,10^3]$ range. We observe that although the performance of constant step-size SGD is sensitive to its step-size; SGD with Armijo line-search is robust with respect to the $c$ parameter. Similarly, we find that Adam is quite robust with respect to its initial learning rate. 

\subsection{Min-max optimization for bilinear games}
\label{app:experiments-games}
Chavdarova et al.~\cite{chavdarova2019reducing} propose a challenging stochastic bilinear game as follows:
\begin{equation}\notag
    \min _{\boldsymbol{\theta} \in \mathbb{R}^{d}} \max _{\boldsymbol{\varphi} \in \mathbb{R}^{d}} \frac{1}{n} \sum_{i=1}^{n}\left(\boldsymbol{\theta}^{\top} \boldsymbol{b}_{i}+\boldsymbol{\theta}^{\top} \boldsymbol{A}_{i} \boldsymbol{\varphi}+\boldsymbol{c}_{i}^{\top} \boldsymbol{\varphi}\right),
    \;\;
    \left[\boldsymbol{A}_{i}\right]_{k l}=\delta_{k l i}\,,\,\left[\boldsymbol{b}_{i}\right]_{k},\left[\boldsymbol{c}_{i}\right]_{k} \sim \mathcal{N}(0,\tfrac{1}{d}), 1 \leq k, l \leq d
\end{equation}
Standard methods such as stochastic extragradient fail to converge on this example. We compare Adam, ExtraAdam~\cite{gidel2018variational}, SEG with backtracking line-search using Equation~\ref{eq:lip-ls} with $c = 1/\sqrt{2}$ and $p$-SVRE~\cite{chavdarova2019reducing}. The latter combines restart, extrapolation and variance reduction for finite sum. It exhibits linear convergence rate but requires the tuning of the restart parameter $p$ and do not have any convergence guarantees on such bilinear problem. ExtraAdam~\cite{gidel2018variational} combines extrapolation and Adam has good performances on GANs although it fails to converge on this simple stochastic bilinear example.  

In our synthetic experiment, we consider two variants of this bilinear game; one where interpolation condition is satisfied, and the other when it is not. As predicted by the theory, SEG + Lipschitz results in linear convergence where interpolation is satisfied and does not converge to the solution when it is not. When interpolation is satisfied, empirical convergence rate is faster than SVRE, the best variance reduced method. Note that SVRE does well even in the absence of interpolation, and the both variants of Adam fail to converge on either example. 

\begin{figure}
    \centering
    \includegraphics[width = \textwidth]{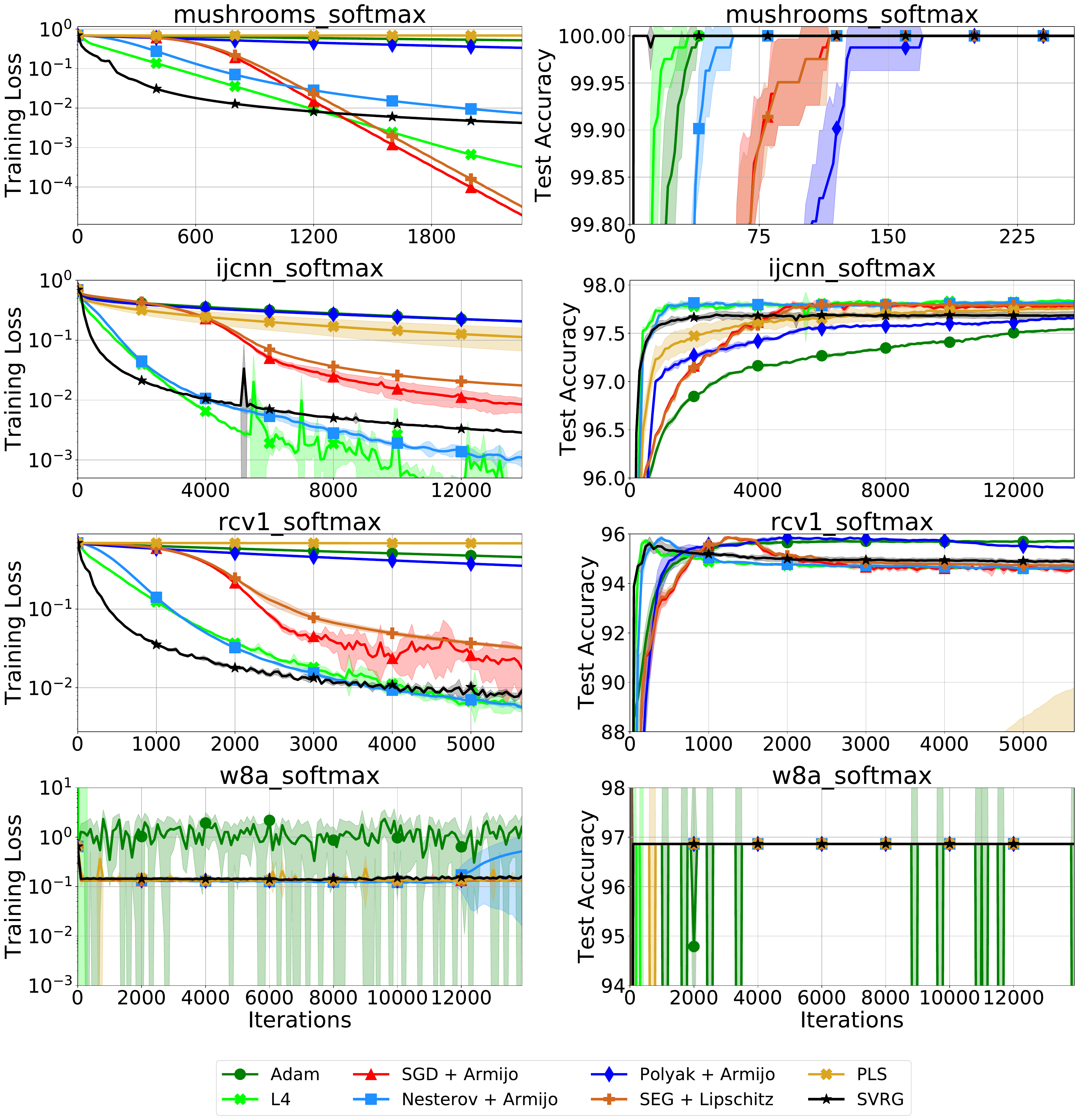}
    \caption{Binary classification using a softmax loss and RBF kernels on the mushrooms, ijcnn, rcv1, and w8a datasets. \textit{Only} the mushrooms dataset satisfies interpolation with the selected kernel bandwidths. We compare against L4 Mom in addition to the other baseline methods; L4 Mom converges quickly on all datasets, but is unstable on ijcnn. Note that w8a dataset is particularly challenging for Adam, which shows large, periodic drops test accuracy. Our line-search methods quickly and stably converge to the global minimum despite the ill-conditioning. }
    \label{fig:additional_kernel_results}
\end{figure}
\begin{figure}
    \centering
    \includegraphics[width = 0.8\textwidth]{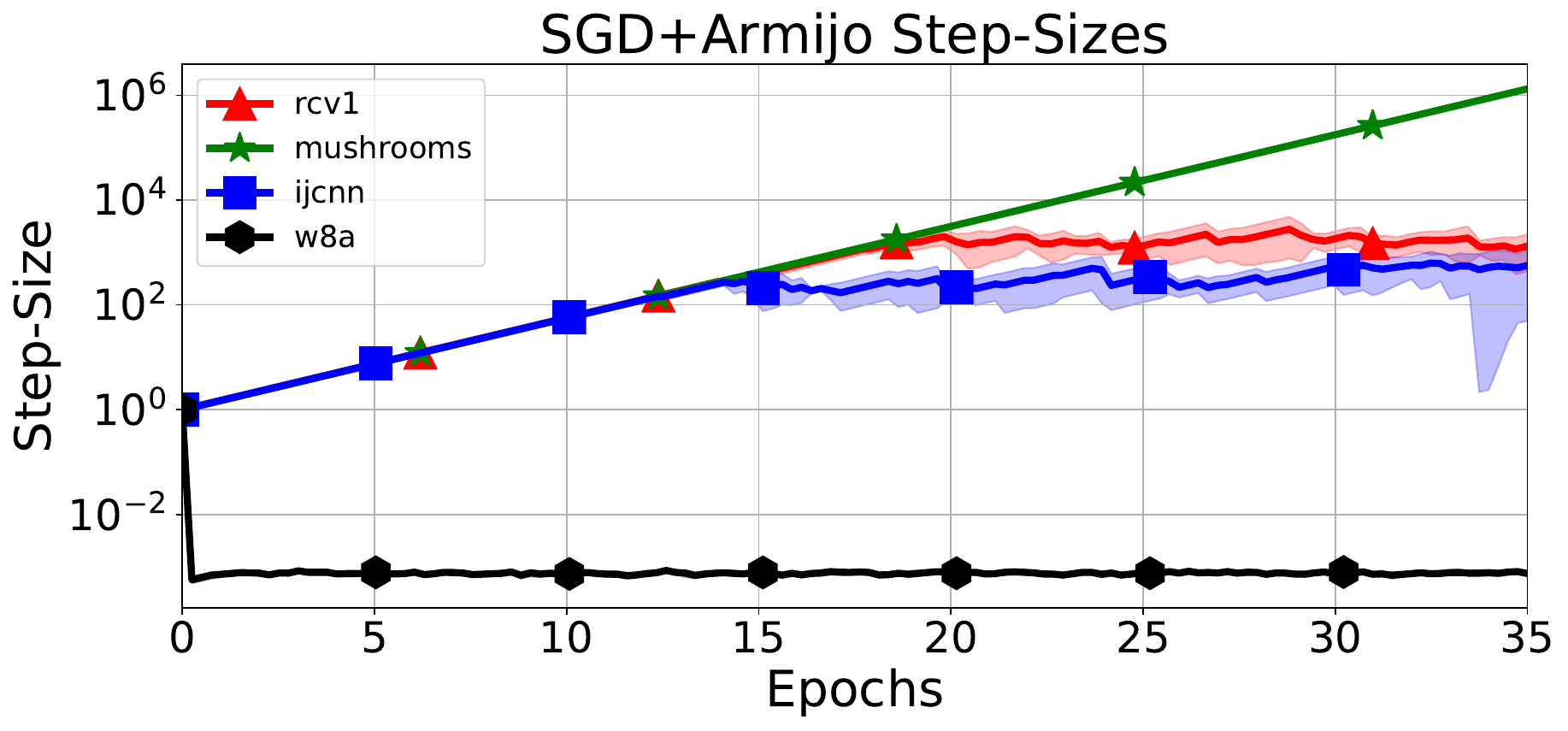}
    \caption{Variation in the step-sizes for SGD + Armijo for binary classification with softmax loss and RBF kernels on the mushrooms, ijcnn, rcv1 and w8a datasets. Recall that we use reset option 2 in Algorithm~\ref{alg:reset_options}. The step-size grows exponentially on mushrooms, which satisfies interpolation. In contrast, the step-sizes for rcv1, ijcnn, and w8a increase or decrease to match the smoothness of the problem.  }
    \label{fig:kernel_step_sizes}
\end{figure}

\begin{figure}
    \centering
    \includegraphics[width = \textwidth]{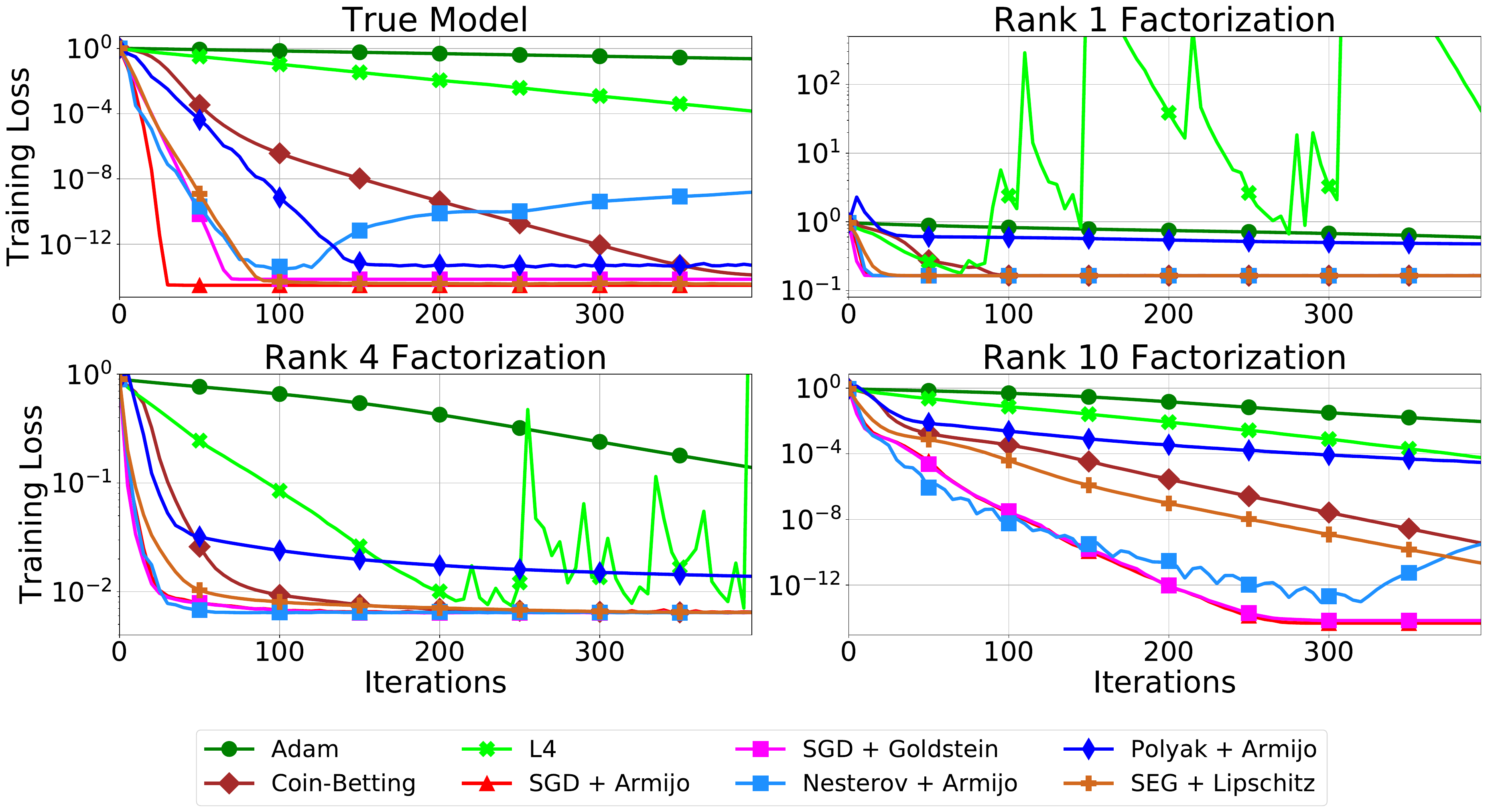}
    \caption{Matrix factorization using the true model and rank $1$, $4$, $10$ factorizations. Rank $1$ factorization is under-parametrized, while ranks $4$ and $10$ are over-parametrized. Only rank $10$ factorization and the true model satisfy interpolation. We include L4 Mom as additional baseline optimizer. L4 Mom is unstable and does not converge on rank $1$ and $4$ factorization, where interpolation is not satisfied; it exhibits slow convergence on the true model and rank 10 factorization. }
    \label{fig:matrix_fac_appendix}
\end{figure}

\subsection{Synthetic Experiment and Binary Classification with Kernels}
\label{app:additional-experiments-kernels}

We provide additional results for binary classification with RBF kernels on the rcv1 and w8a datasets. As before, we do not use regularization. We compare against L4 Mom~\cite{rolinek2018l4} as well as the original baselines introduced in Section~\ref{sec:experiments-setup}. For fairness, we reproduce the results for mushrooms and ijcnn with L4 Mom included. Figure~\ref{fig:additional_kernel_results} shows the training loss and test accuracy for the methods considered, while Figure~\ref{fig:matrix_fac_appendix} shows the evolution of step-sizes for SGD+Armijo on all four kernel datasets. 

The proposed line-search methods perform well on both rcv1 and w8a although neither dataset satisfies the interpolation condition with the selected kernel bandwidths. Furthermore, all of the proposed line-search methods converge quickly and remain at the global minimum for the w8a dataset, which is particularly ill-conditioned. In contrast, adaptive optimizers, such as Adam, fail to converge. Unlike other methods, PLS uses a separate mini-batch for each step of the line-search procedure. Accordingly, we plot the number of iterations \textit{accepted} by the probabilistic Wolfe conditions, which may correspond to several mini-batches of information. Despite this, PLS converges slowly. In practice, we observed that the initial step-size was accepted at most iterations of the PLS line-search. 

Figure~\ref{fig:matrix_fac_appendix} provides an additional comparison against L4 Mom on the synthetic matrix factorization problem from Section~\ref{sec:experiments-synthetic}. We observe that L4 Mom is unstable when used for \textit{stochastic} optimization of this problem, especially when interpolation is not satisfied. The method converges slowly when interpolation is satisfied.

%% file: App-Algorithms.tex
\newpage
\section{Algorithm Pseudo-Code}
\label{app:algorithms}

\begin{figure*}[h]
		\begin{algorithm}[H]
			\caption{\texttt{SGD+Goldstein}($f$, $\x_0$, $\eta_{\max}$, $b$, $c$, $\beta$, $\gamma$)}
			\begin{algorithmic}[1]
			\State $\eta \gets \eta_{\max}$
			\For{$k = 0, \dots, T$}
			    \State $i_k \gets$ sample a minibatch of size $b$ with replacement
			    \While{1} 
			        \If{$\fk \left(\xk - \eta \gradk{\xk} \right) > \fk(\xk) -  c \cdot \eta \normsq{ \gradk{\x_k}  }$}         \Comment{check Equation \eqref{eq:c-ls}}
			            \State $\eta \gets \beta \cdot \eta$    
			        \ElsIf{$\fk \left(\xk - \eta \gradk{\xk} \right) < \fk(\xk) -  (1-c) \cdot \eta \normsq{ \gradk{\x_k}  }$} \Comment{check curvature condition}
			        \State $\eta \gets \min\left\{\gamma \cdot \eta, \eta_{\max} \right\}$    
			        \Else
			            \State break                                             \Comment{accept step-size $\eta$}
			        \EndIf
			    \EndWhile
			    \State $\xkk \gets \xk - \eta \gradk{\xk}$ \Comment{take SGD step with $\eta$}
			\EndFor\\
			\State \Return $\xkk$
			\end{algorithmic}
			\label{alg:SGD_Goldstein}
		\end{algorithm}
		\begin{algorithm}[H]
			\caption{\texttt{SEG+Lipschitz}($f$, $\x_0$, $\eta_{\max}$, $b$, $c$, $\beta$, $\gamma$, \texttt{opt})}
			\begin{algorithmic}[1]
			\State $\eta \gets \eta_{\max}$
			\For{$k = 0, \dots, T$}
			   \State $i_k \gets$ sample a minibatch of size $b$ with replacement
			    \State $\eta \gets$ \texttt{reset}$(\eta, \eta_{\max}, \gamma, b, k, \mbox{\texttt{opt}})$
			    \While{$\norm{\gradk{\xk - \eta \gradk{\xk}} - \gradk{\xk}} > c \; \norm{\gradk{\xk}} $} \Comment{check Equation \eqref{eq:lip-ls}}
			        \State $\eta \gets \beta \cdot \eta$    \Comment{backtrack by a multiple of $\beta$}
			    \EndWhile
			    \State $\xkh \gets \xk - \eta \gradk{\xk}$   \Comment{take SEG step with $\eta$}
			    \State $\xkk \gets \xk - \eta \gradk{\xkh}$ 
			\EndFor\\
			\State \Return $\xkk$
			\end{algorithmic}
			\label{alg:SEG_Armijo}
		\end{algorithm}
		\caption{Pseudo-code for two back-tracking line-searches used in our experiments. \texttt{SGD+Goldstein} implements SGD with the Goldstein line search described in Section \ref{sec:eta-tricks} and \texttt{SEG+Lipschitz} implements SEG with the Lipschitz line-search described in Section \ref{sec:seg}. For both line-searches, we use a simple back-tracking approach that multiplies the step-size by $\beta < 1$ when the line-search is not satisified. We implement the forward search for Goldstein line-search in similar manner and multiply the step-size by $\gamma > 1$. See Algorithm \ref{alg:reset_options} for the implementation of the \texttt{reset} procedure. }
		\label{fig:ls_pseudo_code}
\end{figure*}

\begin{figure*}[h]
		\begin{algorithm}[H]
			\caption{\texttt{Polyak+Armijo}($f$, $\x_0$, $\eta_{\max}$, $b$, $c$, $\beta$, $\gamma$, $\alpha$, \texttt{opt})}
			\begin{algorithmic}[1]
			\State $\eta \gets \eta_{\max}$
			\For{$k = 0, \dots, T$}
			    \State $i_k \gets$ sample a minibatch of size $b$ with replacement
			    \State $\eta \gets$ \texttt{reset}$(\eta, \eta_{\max}, \gamma, b, k, \mbox{\texttt{opt}})$
			    \While{$\fk \left(\xk - \eta \gradk{\xk} \right) > \fk(\xk) -  c \cdot \eta \normsq{ \gradk{\x_k}  }$} \Comment{check Equation \eqref{eq:c-ls}}
			        \State $\eta \gets \beta \cdot \eta$    \Comment{backtrack by a multiple of $\beta$}
			    \EndWhile
			    \State $\xkk \gets \xk - \eta \gradk{\xk} + \alpha (\xk - \x_{k-1})$ \Comment{take SGD step with $\eta$ and Polyak momentum}
			\EndFor\\
			\State \Return $\xkk$
			\end{algorithmic}
			\label{alg:Polyak_Armijo}
		\end{algorithm}
		\begin{algorithm}[H]
			\caption{\texttt{Nesterov+Armijo}($f$, $\x_0$, $\eta_{\max}$, $b$, $c$, $\beta$, $\gamma$, \texttt{opt})}
			\begin{algorithmic}[1]
			\State $\tau \gets 1$                           \Comment{bookkeeping for Nesterov acceleration}
			\State $\lambda \gets 1$
			\State $\lambda_{\text{prev}} \gets 0$\\
			\State $\eta \gets \eta_{\max}$
			\For{$k = 0, \dots, T$}
			    \State $i_k \gets$ sample a minibatch of size $b$ with replacement
			    \State $\eta \gets$ \texttt{reset}$(\eta, \eta_{\max}, \gamma, b, k, \mbox{\texttt{opt}})$
			    \While{$\fk \left(\xk - \eta \gradk{\xk} \right) > \fk(\xk) -  c \cdot \eta \normsq{ \gradk{\x_k}  }$} \Comment{check Equation \eqref{eq:c-ls}}
			        \State $\eta \gets \beta \cdot \eta$    \Comment{backtrack by a multiple of $\beta$}
			    \EndWhile
			    \State $\xkh \gets \xk - \eta \gradk{\xk}$
			    \State $\xkk \gets (1 - \tau) \cdot \xkh + \tau \cdot \xk$      \Comment{Nesterov accelerated update with $\eta$}\\
			    \State \texttt{temp} $\gets \lambda$                \Comment{bookkeeping for Nesterov acceleration}
			    \State $\lambda \gets \left(1 + \sqrt{1 + 4 \lambda_{\text{prev}}^2}\right) / 2$
			    \State $\lambda_{\text{prev}} \gets $ \texttt{temp}
			    \State $\tau \gets \left(1 - \lambda_{\text{prev}}\right) / \lambda$
			\EndFor\\
			\State \Return $\xkk$
			\end{algorithmic}
			\label{alg:Nesterov_Armijo}
		\end{algorithm}
		\caption{Pseudo-code for using Polyak momentum and Nesterov acceleration with our proposed line-search techniques. \texttt{Polyak+Armijo} implements SGD with Polyak momentum and Armijo line-search and \texttt{Nesterov+Armijo} implements SGD with Nesterov acceleration and Armijo line-search. Both methods are described in \ref{sec:acceleration}. See Algorithm \ref{alg:reset_options} for the implementation of the \texttt{reset} procedure.}
		\label{fig:accelerated_pseudo_code}
\end{figure*}

%% file: main.bbl
\begin{thebibliography}{10}

\bibitem{allen2017katyusha}
Zeyuan Allen-Zhu.
\newblock Katyusha: The first direct acceleration of stochastic gradient
  methods.
\newblock In {\em ACM SIGACT Symposium on Theory of Computing}, 2017.

\bibitem{almeida1998parameter}
Lu{\'\i}s Almeida.
\newblock Parameter adaptation in stochastic optimization.
\newblock {\em On-line learning in neural networks}, 1998.

\bibitem{armijo1966minimization}
Larry Armijo.
\newblock Minimization of functions having {L}ipschitz continuous first partial
  derivatives.
\newblock {\em Pacific Journal of mathematics}, 1966.

\bibitem{bach2012optimization}
Francis Bach, Rodolphe Jenatton, Julien Mairal, Guillaume Obozinski, et~al.
\newblock Optimization with sparsity-inducing penalties.
\newblock {\em Foundations and Trends in Machine Learning}, 2012.

\bibitem{bassily2018exponential}
Raef Bassily, Mikhail Belkin, and Siyuan Ma.
\newblock On exponential convergence of {SGD} in non-convex over-parametrized
  learning.
\newblock {\em arXiv preprint arXiv:1811.02564}, 2018.

\bibitem{baydin2017online}
Atilim~Gunes Baydin, Robert Cornish, David~Martinez Rubio, Mark Schmidt, and
  Frank Wood.
\newblock Online learning rate adaptation with hypergradient descent.
\newblock In {\em ICLR}, 2017.

\bibitem{belkin2019does}
Mikhail Belkin, Alexander Rakhlin, and Alexandre~B. Tsybakov.
\newblock Does data interpolation contradict statistical optimality?
\newblock In {\em AISTATS}, 2019.

\bibitem{ben2009robust}
Aharon Ben-Tal, Laurent El~Ghaoui, and Arkadi Nemirovski.
\newblock {\em Robust optimization}.
\newblock Princeton University Press, 2009.

\bibitem{bengio2012practical}
Yoshua Bengio.
\newblock Practical recommendations for gradient-based training of deep
  architectures.
\newblock In {\em Neural networks: Tricks of the trade}. Springer, 2012.

\bibitem{berrada2019training}
Leonard Berrada, Andrew Zisserman, and M~Pawan Kumar.
\newblock Training neural networks for and by interpolation.
\newblock {\em arXiv preprint arXiv:1906.05661}, 2019.

\bibitem{blanchet2019convergence}
Jose Blanchet, Coralia Cartis, Matt Menickelly, and Katya Scheinberg.
\newblock Convergence rate analysis of a stochastic trust region method via
  supermartingales.
\newblock {\em Informs Journal on Optimization}, 2019.

\bibitem{bollapragada2018progressive}
Raghu Bollapragada, Jorge Nocedal, Dheevatsa Mudigere, Hao-Jun Shi, and Ping
  Tak~Peter Tang.
\newblock A progressive batching l-bfgs method for machine learning.
\newblock In {\em ICML}, 2018.

\bibitem{byrd2012sample}
Richard~H Byrd, Gillian~M Chin, Jorge Nocedal, and Yuchen Wu.
\newblock Sample size selection in optimization methods for machine learning.
\newblock {\em Mathematical programming}, 2012.

\bibitem{cevher2018linear}
Volkan Cevher and Bang~C{\^o}ng V{\~u}.
\newblock On the linear convergence of the stochastic gradient method with
  constant step-size.
\newblock {\em Optimization Letters}, 2018.

\bibitem{libsvm}
Chih-Chung Chang and Chih-Jen Lin.
\newblock {LIBSVM}: A library for support vector machines.
\newblock {\em ACM Transactions on Intelligent Systems and Technology}, 2011.
\newblock Software available at \url{http://www.csie.ntu.edu.tw/~cjlin/libsvm}.

\bibitem{chavdarova2019reducing}
Tatjana Chavdarova, Gauthier Gidel, Fran{\c{c}}ois Fleuret, and Simon
  Lacoste{-}Julien.
\newblock Reducing noise in {GAN} training with variance reduced extragradient.
\newblock In {\em NeurIPS}, pages 391--401, 2019.

\bibitem{chen2015solving}
Yuxin Chen and Emmanuel Candes.
\newblock Solving random quadratic systems of equations is nearly as easy as
  solving linear systems.
\newblock In {\em NeurIPS}, 2015.

\bibitem{de2016big}
Soham De, Abhay Yadav, David Jacobs, and Tom Goldstein.
\newblock Big batch {SGD}: Automated inference using adaptive batch sizes.
\newblock {\em arXiv preprint arXiv:1610.05792}, 2016.

\bibitem{defazio2014saga}
Aaron Defazio, Francis Bach, and Simon Lacoste-Julien.
\newblock Saga: A fast incremental gradient method with support for
  non-strongly convex composite objectives.
\newblock In {\em NeurIPS}, 2014.

\bibitem{defazio2018ineffectiveness}
Aaron Defazio and L{\'{e}}on Bottou.
\newblock On the ineffectiveness of variance reduced optimization for deep
  learning.
\newblock In {\em NeurIPS}, pages 1753--1763, 2019.

\bibitem{deylon1993accelerated}
Bernard Delyon and Anatoli Juditsky.
\newblock Accelerated stochastic approximation.
\newblock {\em {SIAM} Journal on Optimization}, 1993.

\bibitem{duchi2011adaptive}
John Duchi, Elad Hazan, and Yoram Singer.
\newblock Adaptive subgradient methods for online learning and stochastic
  optimization.
\newblock {\em JMLR}, 2011.

\bibitem{friedlander2012hybrid}
Michael~P Friedlander and Mark Schmidt.
\newblock Hybrid deterministic-stochastic methods for data fitting.
\newblock {\em SIAM Journal on Scientific Computing}, 2012.

\bibitem{frostig2015regularizing}
Roy Frostig, Rong Ge, Sham Kakade, and Aaron Sidford.
\newblock Un-regularizing: approximate proximal point and faster stochastic
  algorithms for empirical risk minimization.
\newblock In {\em ICML}, 2015.

\bibitem{gidel2018variational}
Gauthier Gidel, Hugo Berard, Ga{\"e}tan Vignoud, Pascal Vincent, and Simon
  Lacoste-Julien.
\newblock A variational inequality perspective on generative adversarial
  networks.
\newblock In {\em ICLR}, 2019.

\bibitem{goodfellow2016nips}
Ian Goodfellow.
\newblock Nips 2016 tutorial: Generative adversarial networks.
\newblock {\em arXiv preprint arXiv:1701.00160}, 2016.

\bibitem{gratton2017complexity}
Serge Gratton, Cl{\'e}ment~W Royer, Lu{\'\i}s~N Vicente, and Zaikun Zhang.
\newblock Complexity and global rates of trust-region methods based on
  probabilistic models.
\newblock {\em IMA Journal of Numerical Analysis}, 38(3):1579--1597, 2017.

\bibitem{harker1990finite}
Patrick~T Harker and Jong-Shi Pang.
\newblock Finite-dimensional variational inequality and nonlinear
  complementarity problems: a survey of theory, algorithms and applications.
\newblock {\em Mathematical programming}, 1990.

\bibitem{he2016deep}
Kaiming He, Xiangyu Zhang, Shaoqing Ren, and Jian Sun.
\newblock Deep residual learning for image recognition.
\newblock In {\em CVPR}, 2016.

\bibitem{huang2017densely}
Gao Huang, Zhuang Liu, Laurens Van Der~Maaten, and Kilian~Q Weinberger.
\newblock Densely connected convolutional networks.
\newblock In {\em CVPR}, 2017.

\bibitem{iusem2019variance}
Alfredo~N Iusem, Alejandro Jofr{\'e}, Roberto~I Oliveira, and Philip Thompson.
\newblock Variance-based extragradient methods with line search for stochastic
  variational inequalities.
\newblock {\em SIAM Journal on Optimization}, 2019.

\bibitem{iusem2017extragradient}
AN~Iusem, Alejandro Jofr{\'e}, Roberto~I Oliveira, and Philip Thompson.
\newblock Extragradient method with variance reduction for stochastic
  variational inequalities.
\newblock {\em SIAM Journal on Optimization}, 2017.

\bibitem{iusem1997variant}
AN~Iusem and BF~Svaiter.
\newblock A variant of korpelevich’s method for variational inequalities with
  a new search strategy.
\newblock {\em Optimization}, 1997.

\bibitem{jain2017accelerating}
Prateek Jain, Sham Kakade, Rahul Kidambi, Praneeth Netrapalli, and Aaron
  Sidford.
\newblock Accelerating stochastic gradient descent for least squares
  regression.
\newblock In {\em COLT}, 2018.

\bibitem{joachims2005support}
Thorsten Joachims.
\newblock A support vector method for multivariate performance measures.
\newblock In {\em ICML}, 2005.

\bibitem{johnson2013accelerating}
Rie Johnson and Tong Zhang.
\newblock Accelerating stochastic gradient descent using predictive variance
  reduction.
\newblock In {\em NeurIPS}, 2013.

\bibitem{juditsky2011solving}
Anatoli Juditsky, Arkadi Nemirovski, and Claire Tauvel.
\newblock Solving variational inequalities with stochastic mirror-prox
  algorithm.
\newblock {\em Stochastic Systems}, 2011.

\bibitem{karimi2016linear}
Hamed Karimi, Julie Nutini, and Mark Schmidt.
\newblock Linear convergence of gradient and proximal-gradient methods under
  the {P}olyak-{{\L}}ojasiewicz condition.
\newblock In {\em Joint European Conference on Machine Learning and Knowledge
  Discovery in Databases}, 2016.

\bibitem{khobotov1987modification}
Evgenii~Nikolaevich Khobotov.
\newblock Modification of the extra-gradient method for solving variational
  inequalities and certain optimization problems.
\newblock {\em USSR Computational Mathematics and Mathematical Physics}, 1987.

\bibitem{kingma2014adam}
Diederik Kingma and Jimmy Ba.
\newblock Adam: {A} method for stochastic optimization.
\newblock In {\em {ICLR}}, 2015.

\bibitem{kleinberg2018alternative}
Robert Kleinberg, Yuanzhi Li, and Yang Yuan.
\newblock An alternative view: When does {SGD} escape local minima?
\newblock In {\em ICML}, 2018.

\bibitem{korpelevich1976extragradient}
GM~Korpelevich.
\newblock The extragradient method for finding saddle points and other
  problems.
\newblock {\em Matecon}, 1976.

\bibitem{krejic2013line}
Nata{\v{s}}a Kreji{\'c} and Nata{\v{s}}a Krklec.
\newblock Line search methods with variable sample size for unconstrained
  optimization.
\newblock {\em Journal of Computational and Applied Mathematics}, 2013.

\bibitem{kushner1995stochastic}
Harold~J Kushner and Jichuan Yang.
\newblock Stochastic approximation with averaging and feedback: Rapidly
  convergent" on-line" algorithms.
\newblock {\em IEEE Transactions on Automatic Control}, 1995.

\bibitem{li2019convergence}
Xiaoyu Li and Francesco Orabona.
\newblock On the convergence of stochastic gradient descent with adaptive
  stepsizes.
\newblock In {\em AISTATS}, 2019.

\bibitem{li2017convergence}
Yuanzhi Li and Yang Yuan.
\newblock Convergence analysis of two-layer neural networks with {ReLU}
  activation.
\newblock In {\em NeurIPS}, 2017.

\bibitem{liang2018just}
Tengyuan Liang and Alexander Rakhlin.
\newblock Just interpolate: Kernel" ridgeless" regression can generalize.
\newblock {\em arXiv preprint arXiv:1808.00387}, 2018.

\bibitem{lin2015universal}
Hongzhou Lin, Julien Mairal, and Zaid Harchaoui.
\newblock A universal catalyst for first-order optimization.
\newblock In {\em NeurIPS}, 2015.

\bibitem{liu2018mass}
Chaoyue Liu and Mikhail Belkin.
\newblock Accelerating stochastic training for over-parametrized learning.
\newblock {\em arXiv preprint arXiv:1810.13395}, 2019.

\bibitem{loizou2017linearly}
Nicolas Loizou and Peter Richt{\'a}rik.
\newblock Linearly convergent stochastic heavy ball method for minimizing
  generalization error.
\newblock {\em arXiv preprint arXiv:1710.10737}, 2017.

\bibitem{loizou2017momentum}
Nicolas Loizou and Peter Richt{\'a}rik.
\newblock Momentum and stochastic momentum for stochastic gradient, newton,
  proximal point and subspace descent methods.
\newblock {\em arXiv preprint arXiv:1712.09677}, 2017.

\bibitem{loshchilov2016sgdr}
Ilya Loshchilov and Frank Hutter.
\newblock {SGDR}: Stochastic gradient descent with warm restarts.
\newblock {\em arXiv preprint arXiv:1608.03983}, 2016.

\bibitem{luo2019adaptive}
Liangchen Luo, Yuanhao Xiong, Yan Liu, and Xu~Sun.
\newblock Adaptive gradient methods with dynamic bound of learning rate.
\newblock In {\em ICLR}, 2019.

\bibitem{ma2018power}
Siyuan Ma, Raef Bassily, and Mikhail Belkin.
\newblock The power of interpolation: Understanding the effectiveness of {SGD}
  in modern over-parametrized learning.
\newblock In {\em ICML}, 2018.

\bibitem{mahsereci2017probabilistic}
Maren Mahsereci and Philipp Hennig.
\newblock Probabilistic line searches for stochastic optimization.
\newblock {\em JMLR}, 2017.

\bibitem{mescheder2017numerics}
Lars Mescheder, Sebastian Nowozin, and Andreas Geiger.
\newblock The numerics of {GAN}s.
\newblock In {\em NeurIPS}, 2017.

\bibitem{nemirovski2004prox}
Arkadi Nemirovski.
\newblock Prox-method with rate of convergence ${O}(1/t)$ for variational
  inequalities with {L}ipschitz continuous monotone operators and smooth
  convex-concave saddle point problems.
\newblock {\em SIAM Journal on Optimization}, 2004.

\bibitem{nemirovski2009robust}
Arkadi Nemirovski, Anatoli Juditsky, Guanghui Lan, and Alexander Shapiro.
\newblock Robust stochastic approximation approach to stochastic programming.
\newblock {\em SIAM Journal on optimization}, 2009.

\bibitem{nesterov2007gradient}
Yu~Nesterov.
\newblock Gradient methods for minimizing composite functions.
\newblock {\em Mathematical Programming}, 2013.

\bibitem{nesterov2013introductory}
Yurii Nesterov.
\newblock {\em Introductory lectures on convex optimization: A basic course}.
\newblock Springer Science \& Business Media, 2004.

\bibitem{nocedal2006numerical}
Jorge Nocedal and Stephen Wright.
\newblock {\em Numerical optimization}.
\newblock Springer Science \& Business Media, 2006.

\bibitem{orabona2017training}
Francesco Orabona and Tatiana Tommasi.
\newblock Training deep networks without learning rates through coin betting.
\newblock In {\em NeurIPS}, 2017.

\bibitem{palaniappan2016stochastic}
Balamurugan Palaniappan and Francis Bach.
\newblock Stochastic variance reduction methods for saddle-point problems.
\newblock In {\em NeurIPS}, 2016.

\bibitem{paquette2018stochastic}
Courtney Paquette and Katya Scheinberg.
\newblock A stochastic line search method with convergence rate analysis.
\newblock {\em arXiv preprint arXiv:1807.07994}, 2018.

\bibitem{plagianakos2001learning}
VP~Plagianakos, GD~Magoulas, and MN~Vrahatis.
\newblock Learning rate adaptation in stochastic gradient descent.
\newblock In {\em Advances in convex analysis and global optimization}.
  Springer, 2001.

\bibitem{polyak1964some}
Boris~T Polyak.
\newblock Some methods of speeding up the convergence of iteration methods.
\newblock {\em USSR Computational Mathematics and Mathematical Physics}, 1964.

\bibitem{polyak1963gradient}
Boris~Teodorovich Polyak.
\newblock Gradient methods for minimizing functionals.
\newblock {\em Zhurnal Vychislitel'noi Matematiki i Matematicheskoi Fiziki},
  1963.

\bibitem{rahimi2017reflections}
Ali Rahimi and Ben Recht.
\newblock Reflections on random kitchen sinks, 2017.

\bibitem{reddi2019convergence}
Sashank~J Reddi, Satyen Kale, and Sanjiv Kumar.
\newblock On the convergence of {Adam} and beyond.
\newblock In {\em ICLR}, 2019.

\bibitem{rolinek2018l4}
Michal Rolinek and Georg Martius.
\newblock {L4:} practical loss-based stepsize adaptation for deep learning.
\newblock In {\em NeurIPS}, 2018.

\bibitem{schapire1998boosting}
Robert~E Schapire, Yoav Freund, Peter Bartlett, Wee~Sun Lee, et~al.
\newblock Boosting the margin: A new explanation for the effectiveness of
  voting methods.
\newblock {\em The annals of statistics}, 1998.

\bibitem{schaul2013no}
Tom Schaul, Sixin Zhang, and Yann LeCun.
\newblock No more pesky learning rates.
\newblock In {\em ICML}, 2013.

\bibitem{schmidt2015non}
Mark Schmidt, Reza Babanezhad, Mohamed Ahmed, Aaron Defazio, Ann Clifton, and
  Anoop Sarkar.
\newblock Non-uniform stochastic average gradient method for training
  conditional random fields.
\newblock In {\em AISTATS}, 2015.

\bibitem{schmidt2017minimizing}
Mark Schmidt, Nicolas Le~Roux, and Francis Bach.
\newblock Minimizing finite sums with the stochastic average gradient.
\newblock {\em Mathematical Programming}, 2017.

\bibitem{schmidt2013fast}
Mark Schmidt and Nicolas~Le Roux.
\newblock Fast convergence of stochastic gradient descent under a strong growth
  condition.
\newblock {\em arXiv preprint arXiv:1308.6370}, 2013.

\bibitem{schoenauer2017stochastic}
Alice Schoenauer-Sebag, Marc Schoenauer, and Mich{\`e}le Sebag.
\newblock Stochastic gradient descent: Going as fast as possible but not
  faster.
\newblock {\em arXiv preprint arXiv:1709.01427}, 2017.

\bibitem{schraudolph1999local}
Nicol Schraudolph.
\newblock Local gain adaptation in stochastic gradient descent.
\newblock 1999.

\bibitem{shang2018guaranteed}
Fanhua Shang, Yuanyuan Liu, Kaiwen Zhou, James Cheng, Kelvin Ng, and Yuichi
  Yoshida.
\newblock Guaranteed sufficient decrease for stochastic variance reduced
  gradient optimization.
\newblock In {\em AISTATS}, 2018.

\bibitem{shao2000rates}
S~Shao and Percy Yip.
\newblock Rates of convergence of adaptive step-size of stochastic
  approximation algorithms.
\newblock {\em Journal of mathematical analysis and applications}, 2000.

\bibitem{soltanolkotabi2018theoretical}
Mahdi Soltanolkotabi, Adel Javanmard, and Jason Lee.
\newblock Theoretical insights into the optimization landscape of
  over-parameterized shallow neural networks.
\newblock {\em IEEE Transactions on Information Theory}, 2018.

\bibitem{sun2016guaranteed}
Ruoyu Sun and Zhi-Quan Luo.
\newblock Guaranteed matrix completion via non-convex factorization.
\newblock {\em IEEE Transactions on Information Theory}, 2016.

\bibitem{sutskever2013importance}
Ilya Sutskever, James Martens, George Dahl, and Geoffrey Hinton.
\newblock On the importance of initialization and momentum in deep learning.
\newblock In {\em ICML}, 2013.

\bibitem{tan2016barzilai}
Conghui Tan, Shiqian Ma, Yu{-}Hong Dai, and Yuqiu Qian.
\newblock {B}arzilai-{B}orwein step size for stochastic gradient descent.
\newblock In {\em NeurIPS}, 2016.

\bibitem{tieleman2012lecture}
Tijmen Tieleman and Geoffrey Hinton.
\newblock Lecture 6.5-rmsprop: Divide the gradient by a running average of its
  recent magnitude.
\newblock {\em Coursera: Neural networks for machine learning}, 2012.

\bibitem{truong2018backtracking}
Tuyen~Trung Truong and Tuan~Hang Nguyen.
\newblock Backtracking gradient descent method for general {C}$^{1}$ functions,
  with applications to deep learning.
\newblock {\em arXiv preprint arXiv:1808.05160}, 2018.

\bibitem{tseng1998incremental}
Paul Tseng.
\newblock An incremental gradient (-projection) method with momentum term and
  adaptive stepsize rule.
\newblock {\em SIAM Journal on Optimization}, 1998.

\bibitem{vaswani2019fast}
Sharan Vaswani, Francis Bach, and Mark Schmidt.
\newblock Fast and faster convergence of {SGD} for over-parameterized models
  and an accelerated perceptron.
\newblock In {\em AISTATS}, 2019.

\bibitem{wen2014robust}
Junfeng Wen, Chun-Nam Yu, and Russell Greiner.
\newblock Robust learning under uncertain test distributions: Relating
  covariate shift to model misspecification.
\newblock In {\em ICML}, 2014.

\bibitem{yadav2017stabilizing}
Abhay Yadav, Sohil Shah, Zheng Xu, David Jacobs, and Tom Goldstein.
\newblock Stabilizing adversarial nets with prediction methods.
\newblock {\em arXiv preprint arXiv:1705.07364}, 2017.

\bibitem{yu2006fast}
Jin Yu, Douglas Aberdeen, and Nicol Schraudolph.
\newblock Fast online policy gradient learning with smd gain vector adaptation.
\newblock In {\em NeurIPS}, 2006.

\bibitem{zeiler2012adadelta}
Matthew Zeiler.
\newblock {ADADELTA}: an adaptive learning rate method.
\newblock {\em arXiv preprint arXiv:1212.5701}, 2012.

\bibitem{zhang2016understanding}
Chiyuan Zhang, Samy Bengio, Moritz Hardt, Benjamin Recht, and Oriol Vinyals.
\newblock Understanding deep learning requires rethinking generalization.
\newblock In {\em ICLR}, 2017.

\bibitem{zhang2017yellowfin}
Jian Zhang and Ioannis Mitliagkas.
\newblock Yellowfin and the art of momentum tuning.
\newblock {\em arXiv preprint arXiv:1706.03471}, 2017.

\end{thebibliography}
